\RequirePackage[l2tabu, orthodox]{nag}

\documentclass[12pt]{article}

\usepackage[parfill]{parskip}
\usepackage[margin=1in]{geometry}
\usepackage[defaultlines=3,all]{nowidow}
\usepackage{afterpage}
\usepackage{framed}

\usepackage[T1]{fontenc}
\usepackage[utf8]{inputenc}
\usepackage[english]{babel}
\usepackage{libertine}             
\usepackage{microtype}
\usepackage{titlesec}
\titleformat{\section}[hang]{\normalfont\Large\bfseries}{\thesection}{1em}{}

\usepackage{amsmath, amssymb, amsthm}
\usepackage{amsfonts}
\usepackage{mathtools}
\usepackage{nicefrac}
\usepackage{bbm}                   
\usepackage{bm}
\usepackage{stackrel}

\usepackage[dvipsnames,table,xcdraw]{xcolor}
\usepackage{pifont}
\usepackage{soul}
\definecolor{stabgreen}{RGB}{120, 190, 130}
\definecolor{staborange}{RGB}{240, 165, 60}
\definecolor{stabred}{RGB}{210, 70, 70}

\usepackage{graphicx}
\usepackage{booktabs}
\usepackage{makecell}
\usepackage{multirow}
\usepackage{tabularx}
\usepackage{setspace}
\usepackage{wrapfig}
\usepackage[labelfont=bf,format=plain,justification=raggedright,%
            singlelinecheck=false]{caption}
\usepackage{subcaption}            
\usepackage{enumitem}
\usepackage{multicol}
\usepackage{adjustbox}
\usepackage{tcolorbox}
\usepackage{minitoc}

\let\origcontentsline\addcontentsline
\newcommand{\nocontentsline}[3]{}
\newcommand{\stoptoc}{\let\addcontentsline\nocontentsline}
\newcommand{\resumetoc}{\let\addcontentsline\origcontentsline}

\usepackage{algorithm}
\usepackage{algpseudocode}          
\floatname{algorithm}{Algorithm}

\algtext*{EndFunction}

\usepackage[sort]{natbib}
\usepackage{hyperref}
\hypersetup{
  colorlinks = true,
  linktoc    = all,
  linkcolor  = black,
  citecolor  = orange,
  urlcolor   = MidnightBlue
}
\usepackage[all]{hypcap}
\usepackage[capitalize, noabbrev]{cleveref}
\crefname{equation}{equation}{equations}
\Crefname{equation}{Equation}{Equations}

\usepackage{thmtools, thm-restate}  

\theoremstyle{plain}
\newtheorem{theorem}{Theorem}[section]
\newtheorem{proposition}[theorem]{Proposition}

\newtheorem{corollary}[theorem]{Corollary}

\theoremstyle{definition}

\newtheorem{assumption}[theorem]{Assumption}

\theoremstyle{remark}
\newtheorem{remark}[theorem]{Remark}

\newcommand{\idpt}{\perp\!\!\!\perp}
\newcommand{\nidpt}{\not\perp\!\!\!\perp}

\newcommand{\EE}{\mathbb{E}}

    \def\cC{\mathcal{C}}  \def\cD{\mathcal{D}}
    \def\cG{\mathcal{G}}  \def\cH{\mathcal{H}}
      
  \def\cN{\mathcal{N}}    \def\cP{\mathcal{P}}
      
    \def\cW{\mathcal{W}}  \def\cX{\mathcal{X}}
\def\cY{\mathcal{Y}}  \def\cZ{\mathcal{Z}}

\def\bfx{\mathbf{x}}  \def\bfy{\mathbf{y}}  \def\bfz{\mathbf{z}}  \def\bfw{\mathbf{w}}
  \def\bfK{\mathbf{K}}  \def\bfL{\mathbf{L}}  \def\bfA{\mathbf{A}}
\def\bfg{\mathbf{g}}    \def\bfM{\mathbf{M}}  \def\bfB{\mathbf{B}}
  \def\bfQ{\mathbf{Q}}    \def\bfc{\mathbf{c}}

\def\bfmu{\boldsymbol{\mu}}

\title{Instrumental and Proximal Causal Inference with Gaussian Processes\footnote{Accepted at UAI 2026}}

\author{%
    Yuqi Zhang\textsuperscript{\rm 1} \qquad
    Krikamol Muandet\textsuperscript{\rm 2} \qquad
    Dino Sejdinovic\textsuperscript{\rm 3,4} \\[1ex]
    Edwin Fong\textsuperscript{\rm 1} \qquad
    Siu Lun Chau\textsuperscript{\rm 4}%
}
\date{
\footnotesize{
    \textsuperscript{\rm 1} Department of Statistics and Actuarial Science, University of Hong Kong, Hong Kong \\
    \textsuperscript{\rm 2} Rational Intelligence Lab, CISPA Helmholtz Center for Information Security, Germany \\
    \textsuperscript{\rm 3} Australian Institute for Machine Learning, Adelaide University, Australia \\
    \textsuperscript{\rm 4} College of Computing \& Data Science, Nanyang Technological University, Singapore \\[2ex]}
}

\begin{document}
\maketitle

\begin{abstract}
Instrumental variable (IV) and proximal causal learning (Proxy) methods are central frameworks for causal inference in the presence of unobserved confounding. Despite substantial methodological advances, existing approaches rarely provide reliable epistemic uncertainty (EU) quantification. We address this gap through a Deconditional Gaussian Process (DGP) framework for uncertainty-aware causal learning. Our formulation recovers popular kernel estimators as the posterior mean, ensuring predictive precision, while the posterior variance yields principled and well-calibrated EU under the standard causal identification assumptions. Moreover, the probabilistic structure enables systematic model selection via marginal log-likelihood optimization. Empirical results demonstrate strong predictive performance alongside informative EU quantification, evaluated via empirical coverage frequencies and decision-aware accuracy–rejection curves. Together, our approach provides a unified, practical solution for causal inference under unobserved confounding with reliable uncertainty.
\end{abstract}

\section{Introduction}

Estimating causal effects from observational data is central to decision-making across a wide range of disciplines. A major challenge, however, is the presence of unobserved confounders, which can bias standard estimators and compromise causal validity. Among the frameworks that address this issue, Instrumental Variables (IV)~\citep{angrist1992effect,stock2003retrospectives} and Proximal Causal Learning (Proxy)~\citep{kuroki2014measurement,tchetgen2020introduction,miao2024confounding} have emerged as foundational tools. Under suitable structural assumptions, both frameworks enable the identification of causal effects despite unobserved confounding, that is, they construct valid estimators from observational data for reasoning about causal estimands.

Recent advances in machine learning have provided a rich class of such estimators, ranging from kernel-based methods~\citep{singh2019,muandet2020dual,mastouri2021proximal} to deep learning approaches~\citep{xu2020learning}. While these estimators have demonstrated strong empirical performance and have been studied extensively from a theoretical perspective~\citep{chen2025towards}, they primarily focus on point estimation of causal effects. As a result, principled predictive uncertainty quantification (UQ)—capturing how confident a model is in its causal estimates—remains comparatively underdeveloped. We argue that modelling predictive uncertainty is nevertheless also a crucial component for reliable causal inference. 

Informative predictive uncertainty is crucial for several reasons. First, it is essential for trustworthy and responsible deployment~\citep{vashney2022trustworthy}, particularly in safety-critical settings where causal estimates inform real-world interventions. Second, predictive uncertainty enables risk-aware decision-making, including conservative and selective deployment~\citep{cortes2016learning, shaker2020aleatoric}, or deferral~\citep{madras2018predict} when available evidence is insufficient. Third, predictive uncertainty plays a central role in downstream tasks such as causal data fusion~\citep{chau2021bayesimp}, causal active learning~\citep{gao2025activecq}, and finding optimal treatment~\citep{aglietti2020causal}, where it governs how evidence is weighted and how limited resources are allocated. 
These considerations become even more critical—and substantially more challenging---in the presence of unobserved confounding, where uncertainty must reflect not only data variability but also incomplete knowledge of the confounders.

Despite its importance, existing approaches to quantifying predictive uncertainty in causal inference with unobserved confounding remain limited. Model-agnostic UQ strategies are typically bootstrap-based and heuristic, and therefore lack a coherent probabilistic interpretation; the uncertainty they quantify often reflects sensitivity to the resampling procedure rather than genuine model confidence. Bayesian approaches offer a more principled alternative, but frequently incur substantial computational and methodological overhead, impose strong parametric assumptions~\citep{lopes2014bayesian}, or rely on artificial data-generating mechanisms~\citep{wang2021quasi}. Moreover, evaluations of uncertainty quality are often restricted to basic coverage-based metrics~\citep{hu2023bayesian}, rather than assessing how effectively uncertainty estimates support downstream decision-making—a perspective that has received increasing attention in the broader uncertainty quantification literature~\citep{shaker2020aleatoric,chau2026quantifying}.

\paragraph{Contributions.}
We propose a Gaussian process (GP)~\citep{rasmussen2003gaussian} framework with principled uncertainty quantification for both the IV and Proxy settings. Building on the observation that learning the unconfounded structural function in both settings reduces to solving a Fredholm integral equation, we leverage the theory of deconditional kernel embeddings~\citep{hsu2019bayesian}, which act as pseudo-inverses of conditional expectation operators, and adapt their GP formulation~\citep{chau2021deconditional} to causal learning under unobserved confounding. The resulting methods, \textbf{GPIV} and \textbf{GPProxy}, form a unified Bayesian nonparametric framework for causal estimation across both settings. 

For estimation, we show that the posterior means of \textbf{GPIV} and \textbf{GPProxy} recover the frequentist estimators underlying widely used kernel-based approaches such as Kernel IV (KIV;~\citet{singh2019}) and Kernel Negative Control (KNC;~\citet{singh2023kernelmethodsunobservedconfounding}). Consequently, our approach inherits their strong modeling performance and established asymptotic guarantees, while the Bayesian formulation additionally equips them with principled predictive uncertainty quantification conventional identification assumptions in IV and proximal models. As an added benefit, our Bayesian framework naturally supports principled model selection via marginal likelihood optimization, especially for the second-stages, yielding strong empirical performance without reliance on ad hoc parameter initialization or extensive cross-validation (CV) commonly used in frequentist approaches.


Our empirical results demonstrate consistently strong predictive performance compared to state-of-the-art baselines, which we attribute to the holistic model selection enabled by our Bayesian formulation. 
To evaluate the quality of the quantified predictive uncertainty, we begin by reporting empirical coverage frequencies, followed by assessing the informativeness of the uncertainty in downstream decision-making tasks, such as selective inference, where we evaluate how effectively uncertainty estimates identify instances for which predictions should be withheld by extending the evaluation framework of \citet{shaker2020aleatoric}.

The remainder of the paper is organized as follows. Section~\ref{sec: prelim} introduces the IV and Proxy setting, together with the necessary background on kernel methods and GPs. Section~\ref{sec: methods} presents our technical contributions. Section~\ref{sec: related works} reviews related work. Section~\ref{sec: exp} provides results of experiments and simulations, and Section~\ref{sec: discussion} concludes the paper.

\section{Preliminary}
\label{sec: prelim}

\paragraph{Notations. } Let $\mathcal X$, $\mathcal Y$, $\mathcal U$, and $\mathcal Z$ denote the spaces of the treatment, outcome, (unobserved) confounder, and instrument, respectively. Uppercase letters $X$, $Y$, $U$, and $Z$ denote random variables taking values in $\mathcal X$, $\mathcal Y$, $\mathcal U$, and $\mathcal Z$, with corresponding distributions $P_X, P_Y,P_U$ and $P_Z$. Lowercase letters denote their realizations. In the Proxy setting, with an abuse of notation, $Z$ instead denotes the treatment proxy \footnote{Note that, with an abuse of notation, \( Z \) in the IV setting and in the proxy setting are completely different variables: in the IV setting, \( Z \) is a proxy for the treatment variable, whereas in the proxy setting, \( Z \) is a proxy for confounders.}, and an additional variable $W\in \mathcal W$, serves as the outcome proxy. 

Let $\cH_\cX$ denote the reproducing kernel Hilbert space (RKHS)~\citep{aronszajn1950theory} for functions on $\cX$ associated with the kernel $k_\cX:\cX\times\cX\to\mathbb{R}$. With an abuse of notation, we define $k_x = k_\cX(x,\cdot)$ as the canonical feature map for $x\in\cX$, and $k_X = k_\cX(X,\cdot)$ as the induced random elements in $\mathcal{H_\mathcal{X}}$ map by the random variable $X$. For simplicity, we denote $k_{x,x'}$ as $k_{\mathcal{X}}(x,x')$ for two test points. Given realizations $\bfx = (x_1,\dots,x_n)^\top$, we define kernel matrices by stacking the kernel elements maps along columns as ${\bf \Psi_{x}} = [k_{x_1},\dots,k_{x_n}]$ and the Gram matrix as $\mathbf{K}_{\bfx\bfx} = {\bf \Psi}_\bfx^\top {\bf \Psi}_\bfx = [k_{x_i, x_j}]_{1\leq i,j\leq n}$. For a given test point $x'$, we write $\bfK_{x'\bfx} = [k_{x', x_1},\dots,k_{x', x_n}]$ as the evaluation vector. RKHSs, kernels, and matrices associated with variables $Z, W$ are defined analogously.  


\subsection{Instrumental Variable}
In the IV regression problem, the goal is to estimate a structural function $f:\cX\to\cY$ under the following structural equation model, also illustrated in Figure~\ref{DAGIV}:
\begin{equation}
Y = f(X) + U, \quad \mathbb{E} [U] = 0, \quad \mathbb{E}[U \mid X] \neq 0,
\label{IVSF}
\end{equation}
where $U$ is a unobserved, mean-zero confounder that influences both the treatment $X$ and outcome $Y$. The presence of this confounding renders direct estimation via standard regression methods invalid, since the conditional expectation $\mathbb{E}[Y\mid X=x]$ no longer coincides with the structural function $f(x)$. Under Pearl’s do-calculus framework~\citep{pearl2009causality}, the structural function corresponds to the average treatment effect (ATE) $$f(x) = \mathbb{E}[Y\mid do(X=x)].$$

To address unobserved confounding, an IV \(Z\) can be introduced to enable identification of the structural function \(f\). 
The key requirement is that the instrument is exogenous with respect to the unobserved confounder, in the sense that \(\mathbb{E}[U \mid Z]=0\); formal conditions are presented in Appendix~\ref{appendix: causal inference}. Taking conditional expectations with respect to $Z$ on both sides of Equation~\ref{IVSF} yields
\begin{align}
   \mathbb{E}[Y \mid Z] = \mathbb{E}[f(X) \mid Z] = \int_{\mathcal{X}} f(X)  dP(X\mid Z), 
   \label{Structure Equation IV}
\end{align}
which defines a Fredholm integral equation of the first kind. A variety of approaches have been proposed to solve this ill-posed inverse problem. For example, \citet{muandet2020dual} formulated this as a convex-concave saddle-point problem. In this work, we instead adopt a deconditioning~\citep{hsuBayesianDeconditionalKernel2019} perspective, which constructs a pseudo-inverse of the conditional expectation operator to recover $f$. Specifically, we build on the GP deconditional formulations of \citet{chau2021deconditional}, which we adapt to the IV setting and describe in detail in Section~\ref{sec: methods}.

\begin{figure}[t!]
    \centering
    \begin{subfigure}[b]{0.45\columnwidth}
        \centering
        \includegraphics[width=\linewidth]{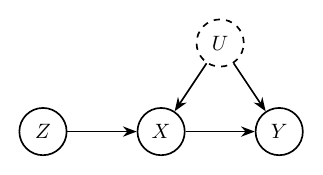}
        \caption{DAG for IV}
        \label{DAGIV}
    \end{subfigure}
    \hfill
    \begin{subfigure}[b]{0.45\columnwidth}
        \centering
        \includegraphics[width=\linewidth]{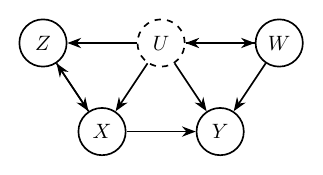}
        \caption{DAG for Proxy}
        \label{DAGProxy}
    \end{subfigure}
    \caption{Causal Graphs}
\end{figure}

\subsection{Proximal Causal Learning }



Compared with the IV setting, the assumptions in the proxy are further relaxed. Instead of assuming the existence of an IV that has no unobserved common cause with the outcome, informally, we assume the treatment proxy $Z$ and the outcome proxy $W$ provide sufficient information about the unobserved confounder $U$, as Figure \ref{DAGProxy} shows.


Our goal remains the estimation of the ATE $\mathbb E[Y \mid \text{do}(X=x)]$, which is again different with the conditional mean $\mathbb E[Y\mid X=x]$ due to the presence of $U$. Under the assumptions in the Appendix \ref{appendix: causal inference}, following the results of \citet{Miao18:Proxy,miao2024confounding}, we first solve for a bridge function $h$ satisfying the Fredholm integral equation:


\begin{equation}
    \mathbb{E}[Y\mid X,Z] = \int_{\cW} h(X,W)\,dP(W\mid X,Z) .
    \label{Bridge Function H}
\end{equation}

Finally, we marginalize over $W$ to identify the ATE:
\begin{equation}
    f(x) = \mathbb{E}[Y\mid \text{do}(X=x)] = \int_{\cW}h(x,W) d P(W)
    \label{Structure Equation for Proxy}
\end{equation}
We note that \citet{mastouri2021proximal,singh2023kernelmethodsunobservedconfounding} also proposed kernel-based approaches to tackle proxy through solving the Fredholm integral equation, similar to IV. We adopt a similar strategy but provide a Bayesian perspective, leveraging the deconditioning GP to recover the ATE from eq.\ref{Bridge Function H}.


\subsection{Kernel embeddings of distributions}




\paragraph{Kernel mean embeddings (KME).} KME~\citep{smola2007hilbert,muandet2016kernel} provides a flexible framework for representing and manipulating probability distributions without requiring parametric assumptions. In particular, the KME for $P_X$ is defined as $\mu_X = \mathbb{E}[k(X,\cdot)]\in\mathcal{H}_\cX$, which can be empirically estimated by $$\hat{\mu}_X =\frac{1}{n}\sum_{i=1}^n k_{x_i}.$$ For a broad class of \emph{characteristic} kernels, such as the Gaussian and Matérn kernels on $\mathbb{R}^d$, this mapping is injective. 



Besides computing the kernel mean, higher-order moments of the induced random element maps can be utilized to model relationships between variables. In particular, the (cross-)covariance operator between variables $X, Z$ is defined as $$C_{XZ} = \mathbb{E}[k_{\cX}(X,\cdot)\otimes k_\cZ(Z, \cdot)] \in \cH_\cX\otimes \cH_\cZ,$$ which can also be estimated via empirical averages, i.e. $$\hat{C}_{XZ} = \frac{1}{n}\sum_{i=1}^n k_{x_i}\otimes k_{z_i} = \frac{1}{n}{\bf \Psi_x\Psi_z^\top }.$$ Similarly, the covariance (or variance) operator $C_{XX}$ and its empirical estimator $\hat{C}_{XX}$ are defined analogously.


\paragraph{Conditional mean embeddings (CME).} Conditional distributions $P_{X\mid Z}$ can also be embedded in the RKHS as $\mu_{X\mid Z=z} = \mathbb{E}[k_X\mid Z=z] \in \cH_\cX$. However, as we often observe samples of $X, Z$ jointly, the standard Monte Carlo estimator for CME is not desirable. Instead, following \citet{song2009hilbert}, it is more common to associate the CME with the conditional mean operator (CMO) $C_{X\mid Z}:\cH_\cZ\to \cH_\cX$, which satisfies $C_{X\mid Z}k_z = \mu_{X\mid Z=z}$ for all $z\in\cZ$. 


Regularity assumptions on the conditional distribution (see \citet{klebanov2020rigorous} for a rigorous treatment) ensure that the CMO is a well-defined bounded operator which admits the representation $C_{X\mid Z} = C_{XZ}C_{ZZ}^{\dagger}$, where $\dagger$ denotes the Moore-Penrose pseudoinverse. A standard estimator of CMO is given by \citet{song2013kernel} as
\begin{equation}
    \hat{C}_{X\mid Z}= \Psi_{\bfx}(\bfK_{\bfz\bfz} + \eta I)^{-1} \Psi^{\top}_{\bfz}\, ,
    \label{CMOEst}    
\end{equation}
which includes a Tikhonov regularization term $\eta>0$. The empirical CME is then $$\hat{\mu}_{X\mid Z=z} = \hat{C}_{X\mid Z}k_z$$ for any $z\in\cZ$.



\paragraph{Deconditional mean embeddings (DME).} DME \citep{hsu2019bayesian} serve as natural counterparts to CME. In CME, we have $$\langle \mu_{X\mid Z=z}, f \rangle_{\cH_{\cX}} = \mathbb{E}[f(X)\mid Z=z]$$ for all $f \in \cH_{\cX}$. DME, denoted $\mu_{X=x\mid Z} \in \cH_{\cZ}$, reverses this process: it recovers the function from the conditional mean via the relation $$\langle\mu_{X=x \mid Z}, \mathbb{E}[f(X)\mid Z=\cdot \,] \,\rangle_{\cH_{\cZ}} = f(x).$$

Analogous to the relationship between the CMO and CME, the deconditional mean operator (DMO) $D_{X|Z}: \cH_{\cX} \rightarrow \cH_{\cZ}$ is associated with DME through the identity $\mu_{X=x\mid Z} = D_{X\mid Z} k_x$ for all $x \in \cX$. This operator acts as a pseudoinverse of CMO, satisfying $$(D_{X|Z})^\top \mathbb{E}[f(X)\mid Z=\cdot] = f$$ for all $f \in \cH_{\cX}$. Under regularity assumptions \citep{hsu2019bayesian}, the DMO can be expressed as a composition of cross-covariance operators and CMOs: $$D_{X\mid Z} = (C_{X\mid Z}C_{ZZ})^{\top}(C_{X\mid Z}C_{ZZ}(C_{X|Z})^\top)^{-1}.$$

DMO offers an alternative way to derive frequentist estimators for the ATE by recovering the true functions from the Fredholm integral equations (eqs. \ref{Structure Equation IV}, \ref{Bridge Function H}). This yields two estimators, which we name DIV and DProxy (detailed in Appendix \ref{appendix: DMO}); however, they do not provide uncertainty quantification.

\section{Methods: Deconditional GP for IV and Proxy}
\label{sec: methods}

In this section, we adopt the deconditioning Gaussian process framework, building on \citet{chau2021deconditional}. That work introduced a Bayesian model whose posterior means recover frequentist DMO-based estimators in statistical downscaling (a connection unrelated to causal inference).\footnote{The term “deconditioning GP” reflects this correspondence; however, neither \citet{chau2021deconditional} nor the present framework relies directly on DMOs.} We demonstrate how this model can be adapted to instrumental variable (IV) and proxy settings, allowing us to recover the structural function $f$ from Eqs.\ref{Structure Equation IV}–\ref{Structure Equation for Proxy} while simultaneously enabling predictive uncertainty quantification. Proofs, technical assumptions, and detailed derivations are deferred to Appendix~\ref{appendix:Proof in sec 3}.



\subsection{GP models for IV}
We place a GP prior $\cG\cP(0, k)$ on $f$, where $k$ is a kernel defined on $\cX$. Given observations $\mathcal{D} = (\bfx,\bfy,\bfz)$, we aim to derive the posterior distribution of $f$ conditioned on $\mathcal{D}$. To this end, we consider the following additive noise model:
\begin{equation}
\label{noise model iv}
\bfy \mid \bfz \sim \mathcal N(g(\bfz),\sigma^2I), \, \, g(z) = \EE_X[f(X) \mid Z=z] .
\end{equation}

We begin with a result on conditional mean processes from~\citet{chau2021deconditional}, which ensures the joint Gaussianity of $f(\bfx)$ and $\bfy$ given $\mathcal{D}$, and adapt it to our IV setting:

\begin{proposition}
\label{prop: CMP_IV}
(CMP for IV.) Under assumptions \ref{assumptioniv1}, \ref{assumptioniv2}, the conditional mean process $\{g(z),z \in \cZ\}$ induced by $f$ with respect to $P_{X\mid Z}$ defined as: 
$$ g(z) =\EE_X\left[f(X)\mid Z=z\right]= \int f(X) \,dP(X\mid Z=z)$$ is a Gaussian Process $g \sim  \cG\cP(0, q) $, with covariance kernel
\begin{equation}
\begin{aligned}
\small
\label{CovKernelq}
q(z, z') = \textnormal{cov}(g(z), g(z'))
= \langle \mu_{X\mid Z=z}, \mu_{X\mid Z=z'} \rangle_{\cH_{\cX}},
\end{aligned}   
\end{equation}
where $\mu_{X\mid Z=z}$ is the CME of $P_{X\mid Z=z}$ in the RKHS $\cH_\cX$.
\end{proposition}


Proposition~\ref{prop: CMP_IV} implies the joint distribution is Gaussian:
\begin{align}
    \begin{bmatrix}
    f(\bfx)\\
    \bfy
    \end{bmatrix}
    \sim 
    \cN\Bigg(
    \begin{bmatrix}
    0 \\
    0
    \end{bmatrix},
    \begin{bmatrix}
    \bfM_{1} & \bfM_2\\
    \bfM_2^\top  & \bfM_3
    \end{bmatrix}
    \Bigg).
\label{GPIV}
\end{align}

We now examine further the covariance matrices, with derivations provided in the Appendix \ref{appendix:Proof in sec 3}. From the GP prior, note that $\textnormal{cov}(f(x),f(x')) = k_{x,x'}$, which gives $$\bfM_1 = \textnormal{cov}(f(\bfx),f(\bfx)) = \bfK_{\bfx\bfx}.$$ Similarly, based on the kernel for $g$ in (eq \ref{CovKernelq}), we define $$\bfQ_{\bfz\bfz} = \textnormal{cov}(g(\bfz),g(\bfz)) = ( C_{X|Z} \Psi_{\bfz})^{\top} C_{X|Z} \Psi_{\bfz}$$ as the kernel matrix induced by $q$ on $\bfz$. Consequently, $$\bfM_3 = \bfQ_{\bfz\bfz} + \sigma^2 I.$$ For $\bfM_2$, the covariance $\textnormal{cov}(f(x),y) = \textnormal{cov}(f(x),g(z))$ is expressed as an inner product of features: $\langle k_x, \mu_{X|Z=z}\rangle_{\cH_\cX}$. This leads to $$\bfM_2 = \textnormal{cov}(f(\bfx), \bfy) = \Psi_{\bfx}^\top C_{X|Z} \Psi_{\bfz}.$$

Finally, using Gaussian conditioning, we obtain:
\begin{proposition}
    Let $\mathbf{\tilde{Q}} =\bfQ_{\bfz\bfz} + \sigma^2 I$. The posterior $f\mid \bfy $ given data $(\bfx, \bfy, \bfz)$ is a GP with mean $\mu$ and covariance $\kappa$:
\begin{align*}
    \mu(x) &= k_x^\top C_{X|Z}\Psi_{\bfz}\mathbf{\tilde{Q}}^{-1}\bfy, \\
    \kappa(x, x') &= k_{x,x'} - k_{x}^\top C_{X|Z}\Psi_{\bfz}\mathbf{\tilde{Q}}^{-1}\Psi_{\bfz}^\top C_{X|Z}^\top k_{x'}.
\end{align*}
\end{proposition}
To obtain a finite data estimate of the mean and covariance, we need to first estimate the CMO/CME, which corresponds to the usual first stage regression in the KIV (\citep{singh2019}), after that we can replace $C_{X|Z}$ with $$\widehat{C}_{X|Z} := \Psi_{\bfx}(\bfK_{\bfz\bfz} + \eta I)^{-1}\Psi_{\bfz}^\top$$ and $\bfQ_{\bfz\bfz}$ with $$\hat{\bfQ}_{\bfz\bfz} = ( \widehat{C}_{X|Z} \Psi_{\bfz})^{\top} \widehat{C}_{X|Z} \Psi_{\bfz}.$$ Therefore, we obtain our empirical estimators for the posterior mean and covariance of the structure function $f$:
\begin{align*}
      \hat{\mu}(x) &= \bfK_{x\bfx}\bfA(\bfA^\top\bfK_{\bfx\bfx}\bfA+ \sigma^2 I)^{-1}\bfy, \\
    \hat{{\kappa}}(x, x') &= k_{x,x'} - \bfK_{x\bfx}\bfA(\bfA^\top\bfK_{\bfx\bfx}\bfA+ \sigma^2 I)^{-1}\bfA^\top\bfK_{\bfx x'}
\end{align*}
where $\bfA = (\bfK_{\bfz\bfz} + \eta I)^{-1} \bfK_{\bfz\bfz}$ is the mediation matrix. 


Both our operator-based deconditioning approach (DIV) and the posterior mean of GP formulation (GPIV) yield closed-form estimators, with the relationship specified by the following proposition.

\begin{proposition}
(Equivalency between KIV, DIV, and GPIV.) Assume $\bfK_{\bfx\bfx}$, $(\bfA^\top\bfK_{\bfx\bfx}\bfA+ \sigma^2 I)$ are invertible, the estimator of KIV, DIV, and the mean posterior of GPIV are equivalent.
\end{proposition}

\subsection{GP models for Proxy}
As in GPIV, we construct a GP model for $f(x)$ and derive its posterior given the observed data $\mathcal{D} = (\bfx,\bfy,\bfz,\bfw)$. We begin by placing a GP prior $\cG\cP(0, k)$ on the bridge function $h$, where $k = k_\cX \cdot k_\cW$ on the product space $\cX \times \cW$. Adopting an additive noise model analogous to (eq \ref{noise model iv}) in the IV setting, we assume
\begin{equation*}
\mathbf{y} \mid \mathbf{x}, \mathbf{z} \sim \mathcal N\bigl(g(\mathbf{x},\mathbf{z}), \sigma^2 I\bigr), \,
\text{where } \, g(x,z) = \mathbb{E}_W\bigl[h(x,W) \mid X=x, Z=z\bigr].
\end{equation*}
We obtain a conditional mean process result similar to that in the previous section.
\begin{proposition}
(Conditional Mean Process for Proxy Settings.)
Let ${g(x,z) : x \in \cX, z \in \cZ}$ be the conditional mean process induced by the bridge function $h$, defined as
$$g(x,z) = \int_{\cW} h(x,W)\,dP(W|X=x,Z=z).$$
Then, under standard regularity assumptions (\ref{assumptionproxy1},\ref{assumptionproxy2}) and by linearity, $g$ follows a GP $g \sim \cG\cP(0, q)$ with covariance kernel
\begin{equation}
q((x,z),(x',z')) = \textnormal{cov}\bigl(g(x,z), g(x',z')\bigr) 
= k_{x,x'}  \bigl\langle \mu_{W\mid X=x,Z=z},\; \mu_{W\mid X=x',Z=z'}\bigr\rangle_ {\mathcal{H}_{\mathcal{W}}}\,.
\label{Covq}
\end{equation}
\end{proposition}
Our primary interest lies in recovering the structural function $f$ appearing in equation (eq \ref{Structure Equation for Proxy}), which requires marginalizing the outcome proxy $W$ from the bridge function $h$. By linearity, $f$ itself will also be a GP:
\begin{proposition}
Define $f(x) = \int h(x,W) dP(W)$. Under mild regularity assumption \ref{assumptionproxy3}, $f$ is almost surely a GP $f \sim \cG\cP(0,r)$ with covariance kernel
\begin{equation}
\label{Covr}
r(x,x')= \textnormal{cov}\bigl(f(x), f(x')\bigr)= k_{x,x'}\lVert \mu_{W} \rVert^{2}_{\mathcal{H}_{\mathcal{W}}} .   
\end{equation}
\end{proposition}
Consequently, in the proxy setting we obtain the following joint Gaussian distribution:
\begin{align}
    \begin{bmatrix}
    f(\bfx)\\
    \bfy
    \end{bmatrix}
    \sim 
    \cN\Bigg(
    \begin{bmatrix}
    0 \\
    0
    \end{bmatrix},
    \begin{bmatrix}
    \bfL_{1} & \bfL_2\\
    \bfL_2^\top  & \bfL_3
    \end{bmatrix}
    \Bigg).
\label{GPProxy}
\end{align}
Based on the covariance kernel specified in (eq \ref{Covr}), we obtain $\bfL_1 = \textnormal{cov}(f(\bfx), f(\bfx)) = c_W \bfK_{\bfx\bfx}$, where $c_W = \lVert \mu_{W} \rVert^{2}_{\mathcal{H}_{\mathcal{W}}}$. 

The matrix $\bfL_3$ consists of two components. The first component originates from the covariance kernel $q$ defined in (eq \ref{Covq}), given by
\begin{equation*}
\mathbf{Q}_{\mathbf{z}\mathbf{w}\mathbf{x}} 
= \operatorname{cov}(g(\mathbf{x},\mathbf{z}), g(\mathbf{x},\mathbf{z})) 
= \mathbf{K}_{\mathbf{x}\mathbf{x}} \odot \left\{ (C_{W\mid X,Z}\Psi_{\mathbf{x},\mathbf{z}})^{\top} (C_{W\mid X,Z}\Psi_{\mathbf{x},\mathbf{z}}) \right\}.
\end{equation*}
Combined with the second component, which accounts for the additive noise variance, we have $\bfL_3 = \bfQ_{\bfz\bfw\bfx} + \sigma^2 I$.
For $\bfL_2$, we derive a similar inner-product expression for the covariance between $f(x)$ and $y$:
\begin{equation*}
\textnormal{cov}(f(x'),y) = \textnormal{cov}(f(x'),g(x,z)) 
= k_{x,x'}\langle \mu_{W}, \mu_{W\mid X=x,Z=z}\rangle _{\mathcal{H}_{\mathcal{W}}},    
\end{equation*}

leading to $\textnormal{cov}(f(\bfx), \bfy) = \bfK_{\bfx\bfx} \odot (\bfmu_{W})^{\top} C_{W\mid X,Z}\Psi_{\bfx,\bfz}$. Here, $\bfmu_{W} = (\mu_{W}, \ldots, \mu_{W})^{\top} = \mu_{W} \mathbf{1}_{n}$ denotes a vector of identical kernel mean embeddings.

Applying Gaussian conditioning yields the following posterior mean and variance expressions:
\begin{proposition}
 Let $\mathbf{\tilde{Q}}=(\mathbf{Q}_{\mathbf{z}\mathbf{w}\mathbf{x}} + \sigma^2I)$ and $\mathbf{\tilde{C}}=\Psi_\mathbf{x} \otimes C_{W\mid X,Z}\Psi_{\mathbf{x},\mathbf{z}}$. The posterior $f\mid \mathcal{D}$ is a GP with mean $\mu$ and covariance $\kappa$:
\begin{align*}
\small 
    \mu(x) &= (k_x \otimes \mu_W)^\top \mathbf{\tilde{C}} \mathbf{\tilde{Q}}^{-1} \mathbf{y} \\
    \kappa(x, x') &= c_W k_{x,x'} - (k_x \otimes \mu_W)^\top \mathbf{\tilde{C}} \mathbf{\tilde{Q}}^{-1} \mathbf{\tilde{C}}^\top(k_{x'} \otimes \mu_W)
\end{align*}  
\end{proposition}
We begin by estimating the CMO and the KME of $W$ from the finite sample $\mathcal{D}$. The CMO is estimated as $\widehat{C}_{W\mid X,Z} = \Psi_\bfw(\bfK_{\bfx\bfx}\odot\bfK_{\bfz\bfz}+\eta I)^{-1}\Psi_{\bfx,\bfz}^{\top}$, which coincides with the first-stage regression in the frequentist kernel regression method \citep{mastouri2021proximal,singh2023kernelmethodsunobservedconfounding}. The kernel mean estimator is obtained via the Monte Carlo average $\hat{\mu}_W = n^{-1}\sum_{i=1}^{n}k_{w_i}$; consequently, the constant appearing in $\bfL_1$ is estimated as $\hat{c}_W = n^{-2}\sum_{1\leq i,j\leq n}k_{w_{i},w_{j}}$. This averaging operation effectively marginalizes $W$ out of the bridge function $h$. We note that if one instead constructs a joint Gaussian distribution over $h(x,w)$ and $g(x,z)$ and then marginalises $W$ from the posterior of $h$, the resulting posterior mean for $f\mid \mathcal{D}$ remains identical.

Therefore, we obtain our empirical estimators for the posterior mean of the structure function, 
\begin{align*}
 \hat{\mu}(x) &= \{\mathbf{K}_{x\mathbf{x}} \odot (\mathbf{K}_{\bar{w}\mathbf{w}} \mathbf{B})\} \{\mathbf{K}_{\mathbf{x}\mathbf{x}} \odot (\mathbf{B}^{\top}\mathbf{K}_{\mathbf{w}\mathbf{w}} \mathbf{B}) + \sigma^2 I\}^{-1} \mathbf{y},\\
 \hat{\kappa}(x, x')& = \hat{c}_W \, k_{x,x'} - \{\mathbf{K}_{x\mathbf{x}} \odot (\mathbf{K}_{\bar{w}\mathbf{w}} \mathbf{B}) \} 
\{\mathbf{K}_{\mathbf{x}\mathbf{x}} \odot (\mathbf{B}^{\top}\mathbf{K}_{\mathbf{w}\mathbf{w}} \mathbf{B}) + \sigma^2 I \} ^{-1}\{\mathbf{K}_{x'\mathbf{x}} \odot (\mathbf{K}_{\bar{w}\mathbf{w}} \mathbf{B}) \}^{\top},
\end{align*}
where $\bfB = (\bfK_{\bfx \bfx}\odot\bfK_{\bfz \bfz} + \eta I)^{-1}(\bfK_{\bfx \bfx}\odot\bfK_{\bfz \bfz})$ acts as the mediation matrix for proxy setting. Moreover, we denote $\mathbf{K}_{\bar{w}\mathbf{w}} = (\langle\hat{\mu}_W,k_{w_1}\rangle,...,\langle\hat{\mu}_W,k_{w_n}\rangle) = \mathbf{K}_{\mathbf{w}\bar{w}}^{\top}$.

Similarly, we state the relationship between the posterior mean of GPProxy and the frequentist kernel methods:

\begin{proposition}
Assuming $\mathbf{K}_{\mathbf{x}\mathbf{x}} \odot (\mathbf{B}^{\top}\mathbf{K}_{\mathbf{w}\mathbf{w}} \mathbf{B})$ is invertible, the posterior mean of GPProxy is equivalent to the estimator produced by KNC \citep{singh2023kernelmethodsunobservedconfounding}.
\end{proposition}

However, our operator-based deconditioning method, DProxy, employs a two-stage regression procedure that differs from GPProxy. Specifically, DProxy is equivalent to the original formulation of KNC (version 1) introduced by \citet{singh2023kernelmethodsunobservedconfounding}, which we refer to as KNC-orig. In this version, the first-stage regression involves the conditional mean operator $C_{X,W|X,Z}$, which does not satisfy the completeness assumptions required for identifiability in proximal causal inference \citep{mastouri2021proximal}. In contrast, KPV \citep{mastouri2021proximal} follows the same regression structure and loss function as GPProxy and KNC, but differs in its representation, leading to a distinct closed-form solution. We refer the reader to Appendix \ref{appendix: ProxyEqv} for further details.

\subsection{Hyperparameter Selection}

In standard GP models, hyperparameters are tuned by maximizing the marginal log‑likelihood (MLL) \citet{williams2006gaussian}. In our IV and proxy settings, the full set of hyperparameters is $\theta_{\text{IV}} = (l_x, l_z, \eta, \sigma^2)$ and $\theta_{\text{proxy}} = (l_x, l_z, l_w, \eta, \sigma^2)$. Ideally, one would optimize all of them jointly via MLL. However, we observe the difficulty of convergence and the fact that joint optimization often drives the length‑scale $l_z$ (or the regularizer $\eta$) to extreme values (e.g., $l_z \to 0$), leading to poor predictive performance (Appendix G.5 for an identifiability analysis). We therefore adopt a two‑stage hybrid strategy that separates heuristic first‑stage choices from principled second‑stage MLL:

Stage 1 (heuristic): We fix the kernel hyperparameters associated with the instrument (i.e., $l_z$ for an RBF kernel and the regularization $\eta$) using the median heuristic \citet{singh2019} and a small constant (e.g., $\eta = 0.1$) after standardizing the inputs, please refer to Appendix \ref{appendix: abstudy lengthscale} - \ref{appendix: eta} for further sensitivity analysis. This choice is not learned from the marginal likelihood; it follows standard practice in the previous kernel methods and prevents the collapse observed under joint MLL.

Stage 2 (MLL‑based): With $l_z$ and $\eta$ fixed, we optimize the remaining hyperparameters ($l_x$, $\sigma^2$, and $l_w$ in the proxy setting) by maximizing the marginal log‑likelihood $\log p(y \mid \theta) \propto -\log |\tilde{Q}| - y^\top \tilde{Q}^{-1} y$, where $\tilde{Q} = Q_{zz} + \sigma^2 I$ for IV and $\tilde{Q} = Q_{zwx} + \sigma^2 I$ for proxy.

This hybrid scheme avoids the identifiability and convergence issues of full MLL while still enjoying the benefits of principled model selection for the second‑stage parameters. Importantly, even optimizing only $l_x$ and $\sigma^2$ via MLL (with $l_z$, $\eta$ fixed heuristically) already yields substantial improvements over non‑Bayesian baselines, which rely on data splitting and cross‑validation. Our method uses all data for both stages without reducing effective sample size (please refer to the Appendix \ref{appendix: split data}).

\section{Related Works}
\label{sec: related works}

\paragraph{IV and Proxy methods.}
A broad range of methods has been developed for the IV setting. Kernel IV~(KIV)~\citep{singh2019} extend the classical linear two-stage least squares framework by accommodating nonlinear relationships via RKHS representations. To mitigate biases inherent to two-stage procedures, alternative single-stage approaches grounded in the generalized method of moments (GMM) have been proposed, including MMRIV~\citep{zhang2023instrumental}, DualIV~\citep{muandet2020dual}, and its Bayesian counterpart QBIV~\citep{wang2021quasi}. Deep learning-based IV estimators have also been investigated \citep{bennett2019deep,xu2020learning}. Other Bayesian formulations \citep{kleibergen2003bayesian,conley2008semi,lopes2014bayesian} have also been studied, but often introduce substantial computational overhead, particularly due to reliance on Markov Chain Monte Carlo (MCMC). More fundamentally, many existing approaches rely on strong and sometimes opaque assumptions about the data-generating process. Although QBIV offers improved scalability, it nevertheless relies on a fictitious data-generating mechanism.


In comparison, methodological developments for proximal causal learning remain relatively limited. Early contributions by \citet{mastouri2021proximal, singh2023kernelmethodsunobservedconfounding} introduced nonlinear estimators based on kernelized two-stage regression, namely KPV and KNC. A GMM formulation tailored to the Proxy setting was also developed in \citet{mastouri2021proximal}. Subsequently, \citet{wu2023doubly} proposed doubly robust kernel-based procedures. Deep learning approaches have further extended both two-stage \citep{xu2021deep} and GMM-based \citep{kompa2022deep} paradigms. To the best of our knowledge, the only Bayesian treatment of proximal causal inference is provided by \citet{hu2023bayesian}, which employs a probit stick-breaking process prior, whereas our approach instead adopts a Gaussian process prior.

Although several studies in IV and Proxy settings mention UQ, few elaborate on its practical utility or significance. Most existing work evaluates UQ solely through coverage properties, whereas our method further explores applications such as Accuracy-Rejection Curve (learning to reject)~\citep{shaker2020aleatoric} and active learning~\citep{aggarwal2014active}. 

\paragraph{Uncertainty quantification (UQ) in causal inference.} UQ lies at the core of causal inference, as causal reasoning fundamentally involves counterfactual quantities for which direct observations are typically unavailable — a difficulty known as the fundamental problem of causal inference. Progress on UQ is particularly visible in partially identified settings, where causal estimands cannot be uniquely point-identified \citep{manski2003partial}. 
Related work has also relaxed structural assumptions, such as additive noise models in instrumental variable settings, leading to set-based uncertainty characterizations expressed as bounds on the structural function \citep{kilbertus2020class}. In contrast, our approach maintains the additive structural assumption while adopting a fully probabilistic model to estimate the ATE.

Another prominent line of research focuses on directly quantifying predictive uncertainty in the estimation of causal estimands. GP models have been incorporated into causal inference to enable principled uncertainty quantification, as explored by \citet{chau2021bayesimp}, \citet{dimitriou2024data}, and \citet{dance2025interventionalprocessescausaluncertainty}. These approaches typically rely on the assumption of no unobserved confounding, under which uncertainty primarily reflects finite-sample variability and model flexibility. ActiveCQ~\citep{gao2025activecq} similarly exploits GP posterior uncertainty for active learning, again under the unconfoundedness assumption. In contrast, our work extends this direction to confounded settings by developing GP-based inference procedures for IV and proxy-based identification strategies.



\section{Experiments}
\label{sec: exp}

\begin{figure*}[t]
\centering
\includegraphics[width=0.95\linewidth]{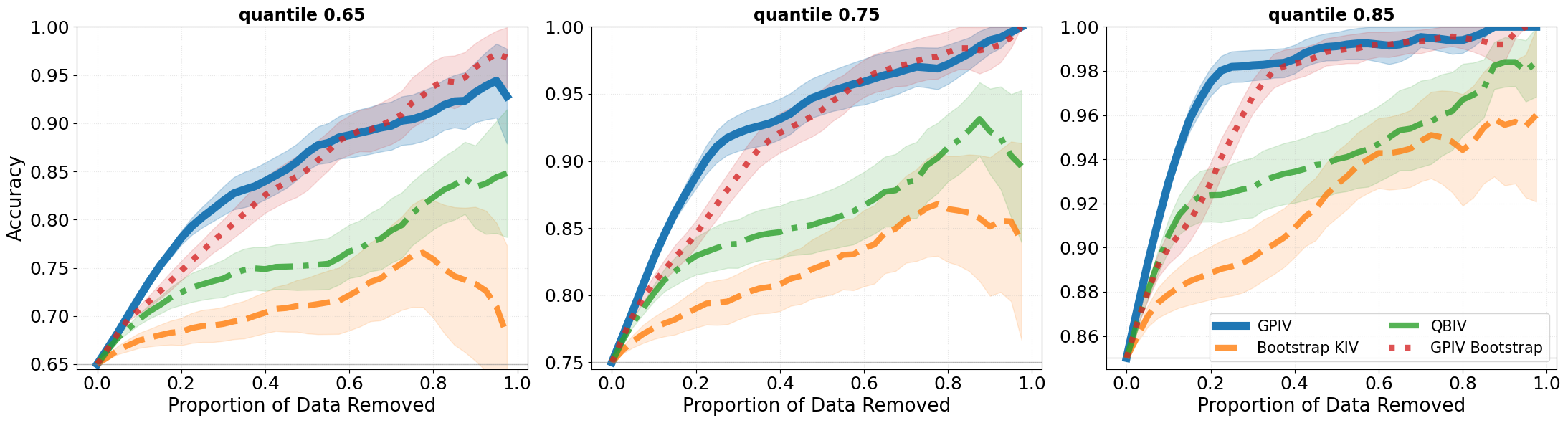}
\caption{Accuracy-rejection curve for log design (data size $=200$), with quantile $= 0.65,0.75,0.85$}
\label{arclog}
\end{figure*}

We evaluate our methods from two dimensions: estimation accuracy and the informativeness of the quantified uncertainty.\footnote{Code is available at https://github.com/YukiZ17/GPIVProxy.}To assess accuracy, we compare our approach with several widely used baselines, particularly kernel-based methods, using a standard metric: the out-of-sample mean squared error (MSE) between the estimated $\hat{f}(x)$ and the ground-truth ATE function $f(x)$. For uncertainty quantification, besides conventional metrics such as nominal $95\%$ coverage rates~\citep{wang2021quasi}, we also assess the informativeness of the predictive uncertainty using the  Accuracy–Rejection Curve (ARC) \citep{hullermeier2022quantification,chau2025integral}, which captures the trade-off between prediction accuracy and selective abstention based on uncertainty estimates. ARC visualizes model accuracy as a function of the rejection rate (i.e., the proportion of instances withheld). By selectively abstaining on the $p\%$ of predictions it is least confident about—in our case, those with the highest posterior variance—the model predicts only on the remaining $(1-p)\%$ of data. Under informative uncertainty, accuracy improves monotonically with the rejection rate, yielding a rising curve. In contrast, uninformative uncertainty or random abstention results in a flat curve. As existing ARC methods are designed for classification problems, we introduce a continuous analogue termed the $\delta$-ARC. Here, a prediction is considered accurate whenever the absolute error between the predicted and true values falls below a tolerance threshold $\delta$, which is set to certain quantiles of the prediction errors over the test data.


\subsection{Simulations on IV}
We compare our method, GPIV, against several representative baselines, including KIV \citep{singh2019}, MMRIV \citep{zhang2023instrumental}, and QBIV \citep{wang2021quasi}, focusing primarily on estimation accuracy. Notably, KIV differs from GPIV in key design choices, particularly in its reliance on data splitting and the use of fixed kernel lengthscales. QBIV, in contrast, represents a Bayesian extension of DualIV \citep{muandet2020dual}. In the context of UQ, we compare GPIV with QBIV, KIV (BS), and QBIV (BS), where the latter two approaches denote that the initial point estimate is used as the mean estimate, and bootstrap (BS) is employed solely for variance estimation. Moreover, \citet{lob2026uniforminferencekernelinstrumental} provide a modified bootstrap method for KIV with good asymptotic behaviour and low computational cost. Nevertheless, their method produces a uniform uncertainty band, which fails to capture the epistemic uncertainty from the data.

\paragraph{Synthetic data.} We adopt the data generating process described in \citet{singh2019}, and evaluate each estimator across three distinct designs of the groundtruth ATE $f$: \text{a sine design } ($f(x) = 2\sin(2\pi x)$) , \text{a linear design } ($f(x)=4x-2$), and \text{a log design } ($f(x)=\log (|16x-8|+1)\times\text{sgn}(x-0.5)$) design. To evaluate model performance, we measure the MSE, coverage, and ARC over 200 evenly spaced points $x \in [0,1]$. 

\textbf{Demand data.} 
We test GPIV on a more challenging airplane ticket demand design used in \citet{muandet2020dual}. An observation consists of $(Y,P,T,S,C)$, where $Y,P,T,S,C$ represents the demand for airline tickets, the price of tickets, the time of year, the sentiment of the customer (discrete variable), and the price shift of the supply cost shifter respectively. Following the notations and methodology in the aforementioned literature, We consider $X=(P,T,S)$ as the input and $Z=(C,T,S)$ as the instruments. The true ATE function is $f(P,T,S) = 100+S(10+P)h(T) - 2P$, where $h(T) = 2( {(T-5)^4}/{600}+\exp(-4(T-5)^2)+{T}/{10} -2)$. Our test data is the set of 3500 points $(p,t,s)$, consisting of: 30 evenly spaced values of $p \in [2.5,27.5]$, 20 evenly spaced values of $t \in [0,10]$, and all 7 discrete values $s \in \{1,\dots,7\}$. 

\textbf{Results.} Table \ref{mseresults} shows that the proposed method, GPIV, consistently achieves the lowest or second-lowest MSE across all baselines. One contributing factor to the observed performance gap is that KIV requires data splitting, which limits its ability to fully exploit the available samples, whereas GPIV leverages the entire dataset. This effect is evident in the synthetic experiments, where KIV with a sample size of 1000 yields performance comparable to GPIV trained on only 500 samples. We further observe that optimizing the lengthscale associated with $X$ substantially improves accuracy, achieving performance comparable to that of deep learning–based methods \citep{xu2020learning} on the demand design. Additional ablation study on lengthscale selection and their impact is provided in Appendix \ref{appendix: abstudy lengthscale}.

\begin{table}[htbp]
\centering
\caption{MSE (with standard error in parentheses) for different IV designs and data sizes $(n)$. The best result in each row highlighted in \textcolor{red}{red} and the second best in \textcolor{orange}{orange}.}
\label{mseresults}
\adjustbox{max width=\textwidth}{
\begin{tabular}{l*5{l}}
\toprule
\multicolumn{1}{c}{Design} & \multicolumn{1}{c}{n} & \multicolumn{1}{c}{\textbf{GPIV}} & \multicolumn{1}{c}{KIV} & \multicolumn{1}{c}{MMRIV} & \multicolumn{1}{c}{QBIV} \\
\midrule
\multirow{3}{*}{sine} 
 & 200 & \textcolor{red}{.163(.014)} & .291(.024) & .199(.015) & \textcolor{orange}{.176(.021)} \\
 & 500 & \textcolor{red}{.084(.009)} & .224(.024) & .189(.011) & \textcolor{orange}{.117(.017)} \\
 & 1000 & \textcolor{red}{.050(.006)} & .210(.021) & .160(.008) & \textcolor{orange}{.087(.012)} \\
\multirow{3}{*}{log} 
 & 200 & \textcolor{orange}{.152(.017)} & .280(.038) & \textcolor{red}{.059(.007)} & .177(.022) \\
 & 500 & \textcolor{orange}{.071(.005)} & .127(.018) & \textcolor{red}{.058(.005)} & .124(.014) \\
 & 1000 & \textcolor{red}{.045(.005)} & .094(.009) & \textcolor{orange}{.052(.004)} & .090(.011) \\
\multirow{3}{*}{linear} 
 & 200 & \textcolor{orange}{.157(.022)} & .250(.031) & \textcolor{red}{.082(.011)} & .193(.024) \\
 & 500 & \textcolor{orange}{.095(.015)} & .120(.017) & \textcolor{red}{.093(.008)} & .126(.016) \\
 & 1000 & \textcolor{red}{.048(.005)} & .077(.015) & .088(.006) & \textcolor{orange}{.069(.007)} \\
\multirow{3}{*}{demand*} 
 & 200 & \textcolor{red}{.070(.003)} & .713(.032) & .649(.025) & \textcolor{orange}{.632(.026)} \\
 & 500 & \textcolor{red}{.053(.002)} & .605(.030) & .582(.014) & \textcolor{orange}{.520(.016)} \\
 & 1000 & \textcolor{red}{.028(.001)} & .474(.026) & .539(.013) & \textcolor{orange}{.429(.015)} \\
\bottomrule
\multicolumn{6}{c}{*Note: We report normalized MSE for demand design.} \\ 
\end{tabular}
}
\end{table}

\begin{table}[h!]
\centering
\caption{Area under quantile $0.75$-ARC (the higher the better) and Coverages (the closer to $0.95$ the better) with standard error in parentheses, with 200 observations each. The best result in each row is highlighted in \textcolor{red}{red} and the second best in \textcolor{orange}{orange}.}
\label{ivuq}
\adjustbox{max width=\textwidth}{
\begin{tabular}{l*5{l}}
\toprule
\multicolumn{1}{c}{Design} & \multicolumn{1}{c}{Metric} & \multicolumn{1}{c}{\textbf{GPIV}} & \multicolumn{1}{c}{KIV(BS)} & \multicolumn{1}{c}{GPIV(BS)} & \multicolumn{1}{c}{QBIV}\\
\midrule
\multirow{2}{*}{sine} 
 & AUC & \textcolor{red}{.854(.015)} & .790(.022) & \textcolor{orange}{.842(.021)} & .746(.025) \\
 & Cov. & \textcolor{red}{.978(.015)} & .738(.082) & .610(.072) & \textcolor{orange}{.876(.137)}  \\
\multirow{2}{*}{log} 
 & AUC & \textcolor{red}{.904(.011)} & .800(.025) & \textcolor{orange}{.893(.012)} & .836(.020)  \\
 & Cov. & \textcolor{orange}{1.00(.000)} & \textcolor{red}{.989(.013)} & .799(.061) & .746(.159) \\
\multirow{2}{*}{linear} 
 & AUC & \textcolor{red}{.905(.012)} & .832(.024) & \textcolor{orange}{.874(.019)} & .848(.020) \\
 & Cov. & \textcolor{orange}{.989(.006)} & \textcolor{red}{.983(.019)} & .779(.052) & .771(.162) \\
\multirow{2}{*}{demand} 
 & AUC & \textcolor{red}{.918(.017)} & .663(.008) & .707(.011) & \textcolor{orange}{.904(.001)} \\
 & Cov. & \textcolor{red}{.961(.007)} & .461(.029) & .625(.059) & \textcolor{orange}{.903(.010)}\\
\bottomrule
\end{tabular}
}
\end{table}

We evaluate the quality of the quantified uncertainty using nominal coverage and ARC. Table~\ref{ivuq} indicates that GPIV outperforms both the bootstrap-based methods and QBIV in terms of coverage.  In particular, the reported $95\%$ confidence intervals of these competing approaches are systematically too narrow, leading to overly optimistic uncertainty estimates. 
Furthermore, the ARC (Figure \ref{arclog}) and the area under the ARC (AUC in Table \ref{ivuq}) indicate that our method provides the most informative uncertainty quantification.

\subsection{Simulation on Proxy}
 We compare our methods GPProxy with several kernel based method: KPV, PMMR \citep{mastouri2021proximal}, PKIPW, and PKDR \citep{wu2023doubly}, KNC-orig (equivalent to DProxy) and KNC \citep{singh2023kernelmethodsunobservedconfounding}. For UQ tasks, we compare GPProxy with GPProxy(BS), which uses variances of bootstrapped posterior means, given that no existing kernel methods in the proximal setting are equipped with uncertainty quantification, to the best of our knowledge.
 
 
\paragraph{Synthetic data.} \quad We introduce the data generating mechanism for the synthetic experiment. The generative process follows that of \cite{wu2023doubly}. We choose 300 evenly spaced points for $x$ in $[-2,4]$ as our test data. 

\paragraph{Demand data.} We follow the proximal version of the airline demand simulation from \citet{xu2021deep} and \citet{kompa2022deep}. The objective is to estimate how airline ticket prices $X$ influence sales $Y$, a relationship confounded by demand $U$ (e.g., seasonal trends). To address this, the fuel cost $Z=(Z_1,Z_2)$ and the number of views on a ticket reservation website $W$ serve as treatment and the outcome proxies, respectively. We evaluate our methods on 300 evenly spaced points for $x$ in $[10,40]$.

\begin{table}
\centering
\caption{MSE (with standard error in parentheses) for different data sizes. Superscripts $s$ and $d$ denote synthetic and demand designs, respectively. The best result in each row is highlighted in \textcolor{red}{red} and the second best in \textcolor{orange}{orange}.}
\label{tab:proxyresults}
\adjustbox{max width=\textwidth}{
\begin{tabular}{c c c c c c}
\toprule
n & \textbf{GPProxy} & DProxy & PMMR & PKDR & KNC \\
\midrule
200$^{s}$ & \textcolor{red}{.385(.041)} & .408(.047) & .421(.048) & \textcolor{orange}{.401(.032)} & .417(.038) \\
500$^{s}$ & \textcolor{orange}{.332(.036)} & .374(.030) & .370(.040) & \textcolor{red}{.319(.028)} & .372(.029) \\
1000$^{s}$ & \textcolor{orange}{.319(.030)} & .354(.035) & .348(.039) & \textcolor{red}{.250(.021)} & .390(.030) \\
200$^{d}$ & \textcolor{orange}{16.8(1.80)} & 19.2(1.47) & 18.3(4.03) & \textcolor{red}{14.4(3.40)} & 18.6(1.32) \\
500$^{d}$ & \textcolor{red}{13.5(.606)} & 17.5(.547) & \textcolor{orange}{17.1(4.01)} & 17.3(3.30) & 17.7(1.41) \\
1000$^{d}$ & \textcolor{red}{12.2(.621)} & 16.1(.551) & \textcolor{orange}{15.4(3.93)} & 15.8(3.36) & 16.2(0.43) \\
\bottomrule
\multicolumn{6}{c}{Note: KPV and PKIPW are significantly worse than others so their } \\
\multicolumn{6}{l}{results are not showing here.} \\
\end{tabular}
}
\end{table}

\begin{table}
\centering
\caption{GPProxy vs GPProxy(BS):Area under quantile $0.75$-ARC and Coverage (Closer to $0.95$ the better) across different data size $(n)$ with standard error in parentheses. The best result in each row is highlighted in \textcolor{red}{red}.} 
\label{tab:proxyuq}
\adjustbox{max width=\textwidth}{
\begin{tabular}{c c c c c c}
\toprule
 &  & \multicolumn{2}{c}{synthetic} & \multicolumn{2}{c}{demand} \\
\cmidrule(r{0.5em}){3-4} \cmidrule(l{0.5em}){5-6}
n & Metric & \textbf{GPProxy} & GPProxy(BS) & \textbf{GPProxy} & GPProxy(BS) \\
\midrule
\multirow{2}{*}{200} & AUC  & \textcolor{red}{.916(.010)} & .842(.018) & \textcolor{red}{.955(.002)} & .754(.024) \\
                     & Cov  & \textcolor{red}{.938(.017)} & .803(.026) & \textcolor{red}{.949(.002)} & .383(.037) \\
\multirow{2}{*}{500} & AUC  & \textcolor{red}{.934(.004)} & .899(.013) & \textcolor{red}{.957(.001)} & .713(.036) \\
                     & Cov  & \textcolor{red}{.929(.014)} & .852(.020) & \textcolor{red}{.989(.007)} & .389(.024) \\
\multirow{2}{*}{1000}& AUC  & \textcolor{red}{.945(.004)} & .922(.007) & \textcolor{red}{.960(.001)} & .704(.038) \\
                     & Cov  & \textcolor{red}{.910(.019)} & .772(.024) & \textcolor{red}{.983(.001)} & .413(.014) \\
\bottomrule
\end{tabular}
}
\end{table}

\paragraph{Results.} Table~\ref{tab:proxyresults} shows that our method achieves predictive accuracy comparable to, or exceeding, that of existing kernel-based approaches. Beyond accuracy, our method delivers superior uncertainty quantification relative to bootstrap-based alternatives. This is evidenced by the ARC (Figures~\ref{arcproxysyn} and~\ref{arcproxydemand}) presented in Appendix~\ref{appendix:ProxyExp}, where our approach exhibits more informative uncertainty behavior. Consistently, Table~\ref{tab:proxyuq} demonstrates that our method attains uniformly better coverage across varying sample sizes compared to bootstrap variance estimates.

\section{Discussion}
\label{sec: discussion}

We introduced a unified GP framework for IV and Proxy. Leveraging the deconditional perspective, our approach recovers widely used frequentist estimators as posterior means while simultaneously quantifying epistemic uncertainty through the posterior variance. This dual interpretation enables principled model selection via marginal likelihood optimization, contributing to strong predictive performance. Empirical results show that the proposed methods attain competitive accuracy while delivering substantially more informative uncertainty estimates, as reflected in both coverage behaviour and selective prediction performance.

Our approach adopts an end-to-end Bayesian framework for inference over the ATE function, where first-stage estimates are treated as parameters defining the likelihood used in the Bayesian update. We discussed a heuristic bootstrap method accounting for the first-stage UQ in the appendix \ref{UQ1}, where it turns out that the first-stage variance does not make a great difference to the total UQ. Eliciting a GP prior for the first stage necessitates more complex inference and optimization procedures, which lie beyond the scope of this work. We leave this extension for future investigation. A further promising direction is to integrate recent axiomatic explainability tools for GP and kernel methods~\citep{chau2022rkhs,chau2023explaining,mohammadi2026amortized,mohammadi2026exact,mohammadi2026quadrashap}, equipping causal estimation with interpretability while propagating predictive uncertainty to the resulting explanations.

\newpage





\bibliography{main}
\appendix
\onecolumn
\section*{Supplementary Material}

\appendix

\section{Causal Inference}
\label{appendix: causal inference}
Here we introduce the omitted assumptions for identifiablility in the IV and Proxy setting. 

\subsection{Instrumental Variable}

\begin{assumption}
Assumptions to identify the ATE under IV setting:
\begin{itemize} 
    \item \emph{Relevance}: $Z$ has a causal influence on $X$, i.e., $X \nidpt Z$.
    \item \emph{Exclusion}: $Z$ does not directly affect $Y$, i.e., $Y \idpt Z \mid (U, X)$.
    \item \emph{Instrumental Unconfoundedness}: $Z$ is independent of the confounder $U$, i.e., $U \idpt Z$.
\end{itemize}
\end{assumption}

\subsection{Proximal Causal Learning}
\begin{assumption}
Assumptions to identify the ATE under Proxy setting \citep{Miao18:Proxy,mastouri2021proximal}:
    \begin{itemize}
    \item \emph{Structural Assumptions}: We assume $Y\idpt Z\mid (X,U)$, and $W\idpt (Z,X) \mid U$
    \item \emph{Completeness Assumptions}: Let $l,g$ be any square integrable functions, then the following holds: 
    \begin{align}
    &\EE[l(U) \mid X=x, Z=z] = 0 \quad \forall (x,z) \in \cX \times \cZ \nonumber 
    \Longleftrightarrow l(U) = 0 \quad \text{a.s.} \\
    &\EE[g(Z) \mid X=x, W=w] = 0 \quad \forall (x,w) \in \cX \times \cW \nonumber 
    \Longleftrightarrow g(Z) = 0 \quad \text{a.s.}
    \end{align}
    \item \emph{Regularity Assumptions}: Please refer to the assumptions (v) - (vii) in the appendix of \citet{Miao18:Proxy}.
\end{itemize}
\end{assumption}

\newpage

\section{Deconditioning Operator Methods}
\label{appendix: DMO}

Here we introduce the methods using the deconditioning mean operator directly. We begin with the empirical estimator of DMO, provided by \citet{hsu2019bayesian}:
\begin{equation}
   \hat{D}_{X\mid Z} = \Psi_{\bfz} (\bfA^{T}\bfK_{\bfx\bfx}\bfA + n\lambda I )\bfA^{T}\Psi_{\bfx}^{T},
\label{DMOEst}
\end{equation}
where $\bfA = (\bfK_{\bfz\bfz} + \eta I)^{-1} \bfK_{\bfz\bfz}$, and $\eta, \lambda$ are the regularization terms for the inversion.

\subsection{Deconditioning IV (DIV)}
Following the Fredholm integral equation (eq \ref{Structure Equation IV}), we further represent the conditional expectation $\EE[f(X)\mid Z] = g(Z) = \EE[Y\mid Z]$, for some $g: \cZ \rightarrow \mathbf{R}$. We assume $g \in \cH_{\cZ}$. Using DMO, we may recover the ATE $f$ from the conditional expectation $g$:$(D_{X|Z})^{T}g = f$. Hence we have:
\begin{align*}
     f(x) &= \langle k_{x},f \rangle _{\cH_{\cX}} =  \langle k_{x}, (D_{X|Z})^{T}g \rangle _{\cH_{\cX}}\\
     &= \langle D_{X|Z}\,k_{x},g \rangle _{\cH_{\cZ}} 
\end{align*}

Assume we observe the data $\cD$ from the IV structure. Then we express the estimation of $f$ by replacing the DMO with its non-parametric estimator:

\begin{align}
    \hat{f}(x) &= \langle \Psi_{\bfz} (\bfA^{T}\bfK_{\bfx\bfx}\bfA + n\lambda I )\bfA^{T}\Psi_{\bfx}^{T}\,k_{x},g \rangle _{\cH_{\cZ}} \\
    &= \bfg^{T} (\bfA^{T}\bfK_{\bfx\bfx}\bfA + n\lambda I )\bfA^{T}\Psi_{\bfx}^{T}\,k_{x}\\
    &= \bfy^{T}  (\bfA^{T}\bfK_{\bfx\bfx}\bfA + n\lambda I )\bfA^{T} \bfK_{\bfx x}\\
    & = \bfK_{x\bfx} \bfA (\bfA^{T}\bfK_{\bfx\bfx}\bfA + n\lambda I ) \bfy
    \label{DMOIV}
\end{align}

where $\bfg = (g(z_{1}),...,g(z_{n})) =\langle \Psi_{\bfz}\,,g \rangle _{\cH_{\cZ}}$. In the second last step, we estimate $g(z_{i}) = \EE[Y|Z=z_{i}]$ with its point estimate $y_{i}$. Denote $\bfy = (y_{1},...,y_{n})$, then $\bfg$ is replaced with $\bfy$, which follows the same manner in \cite{hsu2019bayesian}.

Note that, the last step here can be seen as a kernel regression from $Y$ to $Z$: 
$Y = g(Z) +e$, where $g \in \cH_{\cZ}$ and $e$ is a demeaned error term independent of $Z$.


\subsection{Deconditioning Proxy (DProxy)}
Following the similar method in the previous section, we represent $g(X,Z) = \EE[Y|X,Z]$ in the (eq \ref{Bridge Function H}). The first step is to recover the bridge function $h$ by the DMO
\begin{equation}
    h(x,w) =  \langle k_{x} \otimes k_{w}, (D_{X,W|X,Z})^{T}g \rangle _{\cH_{\cX} \times \cH_{\cW}}.  
    \label{recoverh}
\end{equation}

The second step is to integrate $U$ out:
\begin{align*}
    f(x) &= \int h(x,W) dP(W)\\
    &= \int \langle D_{X,W|X,Z} \,(k_{x} \otimes k_{W}), g \rangle _{\cH_{\cX} \times \cH_{\cZ}}   dP(W)\\
    &= \langle \int D_{X,W|X,Z} \,(k_{x} \otimes k_{W}) dP(W), g \rangle _{\cH_{\cX} \times \cH_{\cZ}} \\
    &= \langle  D_{X,W|X,Z} \,(k_{x} \otimes \mu_{W}), g \rangle _{\cH_{\cX} \times \cH_{\cZ}}
\end{align*}

Now we replace the operators with their estimators:
\begin{align}
\label{DProxy}
    \hat{f}(x) &= \langle  \hat{D}_{X,W|X,Z} \,(k_{x} \otimes \hat{\mu}_{W}), g \rangle _{\cH_{\cX} \otimes \cH_{\cZ}} \notag\\
    & = \langle (\Psi_{\bfx}\otimes\Psi_{\bfz}) [\bfB^{T}(\bfK_{\bfx\bfx}\odot\bfK_{\bfw\bfw})\bfB + n\lambda I ]^{-1}\bfB^{T}(\Psi_{\bfx}\otimes\Psi_{\bfw})^{T}\,(k_{x}\,\otimes\,\frac{1}{n}\sum_{i=1}^{n} k_{w_{i}}),g \rangle _{\cH_{\cX} \otimes\cH_{\cZ}} \notag\\
    &= \langle (\Psi_{\bfx}\otimes\Psi_{\bfz}),g\rangle _{\cH_{\cX} \otimes\cH_{\cZ}} [\bfB^{T}(\bfK_{\bfx\bfx}\odot\bfK_{\bfw\bfw})\bfB + n\lambda I ]^{-1}\bfB^{T} (\bfK_{\bfx x} \odot \bfK_{\bfw \Bar{w}})\notag\\
    &= (\bfK_{x\bfx} \odot \bfK_{\Bar{w}\bfw})\bfB[\bfB^{T}(\bfK_{\bfx\bfx}\odot\bfK_{\bfw\bfw})\bfB + n\lambda I^{-1}] \bfy   
\end{align}
\newpage

\section{Proof and Omitted Assumptions in Section 3}
\label{appendix:Proof in sec 3}
\subsection{GP for IV}

\begin{assumption}
\label{assumptioniv1}
 $f \sim \cG\cP(0, k)$ to be integrable: $\int_\cX |f(X)| \,dP(X\mid Z=z) < \infty$   
\end{assumption}

\begin{assumption}
\label{assumptioniv2}
The kernel $k$ satisfies $$\EE[|k_X|_{\cH_\cX}] = \int_{\cX'}\int_{\cX} \sqrt{\mathrm{cov}(f(x),f(x'))}dP(X)dP(X')< \infty.$$
\end{assumption}

\begin{proposition}
(CMP for IV) Under the assumptions (\ref{assumptioniv1}, \ref{assumptioniv2}), the conditional mean process $\{g(z),z \in \cZ\}$ induced by $f$ with respect to $P_{X\mid Z}$ defined as: 
$$ g(z) =\EE_X\left[f(X)\mid Z=z\right]= \int f(X) \,dP(X\mid Z=z)$$ is a Gaussian Process $g \sim  \cG\cP(0, q) $, with covariance kernel
\begin{equation*}
q(z, z') = \textnormal{cov}(g(z), g(z'))
= \langle \mu_{X\mid Z=z}, \mu_{X\mid Z=z'} \rangle_{\cH_{\cX}},
\end{equation*}
where $\mu_{X\mid Z=z}$ denotes the CME of $P_{X\mid Z=z}$ in the reproducing kernel Hilbert space $\cH_\cX$.
\end{proposition}

\begin{proof}
The proof consists of two parts.
\begin{itemize}
    \item CMP, the conditional expectation of a GP remains a valid GP with proper (finite) mean and covariance. This follows the results form~\citet{chau2021deconditional}.
    \item The mean and covariance kernel are $0$ and $q$:
    We first show the mean of the GP is $0$: 
    \begin{equation*}
        \EE_g[g(z)] = \EE_{X}[\EE_f[f(X)]\mid Z=z] = 0.
    \end{equation*}
    For covariance kernel, we have,
    \begin{align*}
        q(z, z') &= \mathrm{cov}(g(z),g(z'))\\
        &= \iint \mathrm{cov}(f(x),f(x'))dP(X\mid Z=z)dP(X'\mid Z=z')\\
        &= \EE_{X,X'}[k(X, X')|Z=z, Z=z']\\
        &= \langle\mu_{X|Z=z}, \mu_{X|Z=z'}\rangle_{\cH_{\cX}},
    \end{align*}
where the Fubini's theorem is applied to the second step, given the regularity assumption stated in the proposition.    
\end{itemize}
\end{proof}

\textbf{Covariance Matrices in the joint Gaussian of $\bfy$ and $f(\bfx)$ in (eq \ref{GPIV})} 
Firstly, we have $\bfM_1 = \text{cov}(f(\bfx),f(\bfx)) = \bfK_{\bfx \bfx}$ as $\text{cov}(f(x_i),f(x_j)) = k_{x_i,x_j} = [\bfM_1]_{ij}$

Now, for $\bfM_2$, we have:
\begin{align*}
    \text{cov}(f(x),y) &= \text{cov}(f(x),g(z))\\
    &= \int \text{cov}(f(x),f(X))dP(X\mid Z=z)\\
    & = \int \langle k_x,k_X \rangle_{\cH_\cX}dP(X\mid Z=z)\\
    &= \langle k_x,\int k_X dP(X\mid Z=z) \rangle_{\cH_\cX}\\
    & = \langle k_x, \mu_{X|Z=z}\rangle_{\cH_\cX}.
\end{align*}

Therefore, $\mathrm{cov}(f(\bfx),\bfy) = \Psi_{\bfx}^{\top}C_{X|Z}\Psi_{\bfz}$, where 
$$
C_{X|Z}\Psi_{\bfz} = [\mu_{X|Z=z_1},...,\mu_{X|Z=z_n}]^\top.
$$

For $\bfM_3 = \bfQ_{\bfz\bfz} + \sigma^2 I$, since we know the additive noise is independent with $g$, therefore their variance should be additive as well. Therefore we only need to look into the covariance matrix of $g(\bfz)$: 
$\bfQ_{\bfz\bfz} = (C_{X|Z}\Psi_{\bfz})^\top (C_{X|Z}\Psi_{\bfz})$ as 
$$\mathrm{cov}(g(z_i),g(z_j)) = \langle \mu_{X|Z=z_i} , \mu_{X|Z=z_j}\rangle_{\cH_\cX} = q(z_i,z_j) = [\bfQ_{\bfz\bfz}]_{i,j}.$$
Therefore, by Gaussian conditioning, we have 
\begin{equation}
\label{GPIV Conditioning}
    f(\bfx)\mid \bfy \sim \cN(\bfM_2 \bfM_3^{-1}\bfy , \bfM_1-\bfM_2 \bfM_3^{-1} \bfM_2^{\top}).
\end{equation}

\begin{proposition}
    Let $\mathbf{\tilde{Q}} =\bfQ_{\bfz\bfz} + \sigma^2 I$. The posterior $f\mid \bfy $ given data $(\bfx, \bfy, \bfz)$ is a GP with mean $\mu$ and covariance $\kappa$:
\begin{align*}
    \mu(x) &= k_x^\top C_{X|Z}\Psi_{\bfz}\mathbf{\tilde{Q}}^{-1}\bfy, \\
    \kappa(x, x') &= k_{x,x'} - k_{x}^\top C_{X|Z}\Psi_{\bfz}\mathbf{\tilde{Q}}^{-1}\Psi_{\bfz}^\top C_{X|Z}^\top k_{x'}.
\end{align*}
\end{proposition}

\begin{proof}
\label{proofIV}
   Since (eq \ref{GPIV Conditioning}) holds for any $\bfx$, by Kolmogorov extension theorem, we have the posterior of $f$ given the observations is a Gaussian Process.

   For test points $\tilde{\bfx}=(x,x')$, $f(\tilde{\bfx})|\bfy$ follows the same expression in (eq \ref{GPIV Conditioning}), so in this case, we have $\mu(x) = [\bfM_2]_{1,:} \bfM_3^{-1}\bfy$, $\kappa(x, x')=[\bfM_1]_{1,2} - [\bfM_2]_{1,:}\bfM_3^{-1}[\bfM_2]^{T}_{:,2}$, where $[\bfM_1]_{1,2} = k_{x,x'}$ is the covariance between $f(x)$ and $f(x')$, $[\bfM_2]_{1,:} = k_x^\top C_{X|Z}\Psi_{\bfz}$, $[\bfM_2]_{2,:} = k_{x'}^\top C_{X|Z}\Psi_{\bfz}$ represents the first and second row of the covariance matrix $\mathrm{cov}(f(\tilde{\bfx}),\bfy)$, where 
   $$
    k_x^\top C_{X|Z}\Psi_{\bfz} = [\langle k_x,\mu_{X|Z=z_1}\rangle_{\cH_\cX},...,\langle k_x,\mu_{X|Z=z_n}\rangle_{\cH_\cX}].
   $$
   $\bfM_3=\bfQ_{\bfz\bfz} + \sigma^2 I$ does not change as it only depends on the covariance of $\bfy$, which is irrelevant to the new test data points $x,x'$.   
\end{proof}

\subsection{GP for Proxy}
\begin{assumption}
\label{assumptionproxy1}
 $P(\cdot|x,z)$ is well-defined for all $(x,z) \in \mathcal{X} \times \mathcal{Z}$.   
\end{assumption}

\begin{assumption}
\label{assumptionproxy2}
$h\rightarrow\int_{\mathcal{W}}h(x,W)dP(W\mid X=x,Z=z)$ is a bounded linear functional for every $(x,z) \in \mathcal{X} \times \mathcal{Z}$.
\end{assumption}

\begin{proposition}
(CMP for Proxy) Let $\{g(x,z),x\in \cX,z\in \cZ\}$ be the CMP induced by the bridge function $h$ 
$$g(x,z) = \int_{\cW} h(x,W)\,dP(W|X=x,Z=z).$$ 
Then, under the assumption of regularity (Assumptions \ref{assumptionproxy1},\ref{assumptionproxy2}), by linearity, $g \sim \cG\cP(0, q)$ with the covariance
\begin{equation*}
   q((x,z),(x',z')) = \mathrm{cov}(g(x,z),g(x',z')) = k_{x,x'} \langle \mu_{W\mid X=x,Z=z}, \mu_{W\mid X=x',Z=z'}\rangle _{\mathcal{H}_{\mathcal{W}}}.
\end{equation*}
\end{proposition}

\begin{proof}
The joint Gaussianity is guaranteed as the integration is a linear operation, we refer readers to the proof of Gaussianity to \citet{kanagawa2025gaussian}, Section 4. Moreover, we have
$$
\EE_g[g(x,z)] =  \EE_{h}\EE_{W}[h(x,W)|X=x,Z=z] = \EE_{W}[\EE_{h}[h(x,W)]|X=x,Z=z] = 0.
$$
For the covariance kernel, we have 
\begin{align*}
    q((x,z),(x',z')) &= \mathrm{cov}(g(x,z),g(x',z'))\\  
   &=\int \int \mathrm{cov}(h(x,W),h(x',W'))dP(W|X=x,Z=z)dP(W'\mid X=x',Z=z')\\
   &=\int \int k_{x,x'}\,k_{W,W'} \, dP(W|X=x,Z=z)dP(W'\mid X=x',Z=z')\\
   &= k_{x,x'} \langle \mu_{W|X=x,Z=z}, \mu_{W|X=x',Z=z'}\rangle _{\cH_{\cW}}.
\end{align*}

\end{proof}

\begin{assumption}
\label{assumptionproxy3}
 $h\rightarrow\int_{\mathcal{W}}h(x,W)dP(W)$ is a bounded linear functional.   
\end{assumption}

\begin{proposition}
Let $f(x) = \int h(x,W) dP(W)$. Under mild regularity assumption \ref{assumptionproxy3}, $f \sim \cG\cP(0,r)$ a.s. with the covariance kernel:
\begin{equation*}
    r(x,x')= \textnormal{cov}\bigl(f(x), f(x')\bigr)= k_{x,x'}\lVert \mu_{W} \rVert^{2}_{\mathcal{H}_{\mathcal{W}}} .
\end{equation*}
\end{proposition}

\begin{proof}
\label{IVPostProof}
Similarly, The joint Gaussianity is guaranteed as the integration is a linear operation \citep{kanagawa2025gaussian}. Moreover, we have
$$\EE_{f}[f(x)] = \EE_{h}\EE_{W}[h(x,W)] = \EE_{W}[\EE_{h}[h(x,W)]\,] = 0 ,$$ 

and the covariance kernel,

\begin{align*}
    r(x,x') &= \mathrm{cov}(f(x),f(x')) \\
   &=\int_{\mathcal{W'}} \int_{\mathcal{W}} \mathrm{cov}(h(x,W),h(x',W'))dP(W)dP(W')\\
   &=\int_{\mathcal{W'}} \int_{\mathcal{W}} k_{x,x'}\,k_{W,W'} \, dP(W)dP(W')\\
   &= k_{x,x'} \langle \mu_{W}, \mu_{W'}\rangle _{\cH_{\cW}} =  k_{x,x'}\lVert \mu_{W} \rVert^{2}_{\mathcal{H}_{\mathcal{W}}} .
\end{align*}

Note that, Fubini's theorem is used in the above derivations, which is guaranteed with the regularity assumptions.
\end{proof}

\textbf{Covariance Matrices in the Joint Gaussian of $\bfy$ and $f(\bfx)$ in (eq \ref{GPProxy})} 

 $\bfL_1 = \mathrm{cov}(f(\bfx),f(\bfx)) = c_W \bfK_{\bfx\bfx}$ comes from the covariance kernel $r$ between $f(\bfx)$, where $c_W =  \langle \mu_{W}, \mu_{W'}\rangle _{\cH_{\cW}} =\lVert \mu_{W} \rVert^{2}_{\mathcal{H}_{\mathcal{W}}}= \iint\langle k_{W}, k_{W'} \rangle dP(W)dP(W')$.

 For $\bfL_2$, we have,
 \begin{align*}
     \mathrm{cov}(f(x'),y) &= \mathrm{cov}(f(x'),g(x,z)) \\
&= \mathrm{cov}( \int h(x',W'))dP(W') , \int h(x,W)dP(W\mid X=x,Z=z)   \\
& = \iint  \mathrm{cov}(h(x',W'),h(x,W))\,dP(W')dP(W\mid X=x,Z=z)\\
& = \iint  k_{x,x'} \langle k_{W'}, k_W\rangle \,dP(W')dP(W\mid X=x,Z=z)               \\
&= k_{x,x'}\langle \mu_{W}, \mu_{W\mid X=x,Z=z}\rangle _{\mathcal{H}_{\mathcal{W}}}.
 \end{align*}

Note that, we used multiple times of Fubini's theorems in the above derivations. Therefore, we have $\bfL_2 = \mathrm{cov}(f(\bfx),\bfy) = \bfK_{\bfx\bfx} \odot (\bfmu_{W})^{\top}C_{W\mid X,Z}\Psi_{\mathbf{x},\mathbf{z}}$, where $\bfmu_{W} = (\mu_{W},...,\mu_{W})^{\top} = \mu_{W} 1_{n}$, where we again abuse the notation:
$$C_{W\mid X,Z}\Psi_{\mathbf{x},\mathbf{z}} = [\mu_{W|X=x_1,Z=z_1},...,\mu_{W|X=x_n,Z=z_n}]^\top.$$

For $\bfL_3$, similar to $\bfM_3$, the covariance consists of the variance of $g(\bfx,\bfz)$, which is $\mathbf{Q}_{\mathbf{z}\mathbf{w}\mathbf{x}}$, and the independent additive noise $\sigma^2 I$. The derivation is very similar to the $\bfM_3$ and we omitted it here.

Therefore, Gaussian conditioning gives us

\begin{equation}
\label{GPProxy Conditioning}
    f(\bfx)\mid \bfy \sim \cN (\bfL_2 \bfL_3^{-1}\bfy ,\bfL_1- \bfL_2 \bfL_3^{-1} \bfL_2^{\top}).
\end{equation}

\begin{proposition}
 Let $\mathbf{\tilde{Q}}=(\mathbf{Q}_{\mathbf{z}\mathbf{w}\mathbf{x}} + \sigma^2I)$ and $\mathbf{\tilde{C}}=\Psi_\mathbf{x} \otimes C_{W\mid X,Z}\Psi_{\mathbf{x},\mathbf{z}}$. The posterior $f\mid \mathcal{D}$ is a GP with mean $\mu$ and covariance $\kappa$:
\begin{align*}
\small 
    \mu(x) &= (k_x \otimes \mu_W)^\top \mathbf{\tilde{C}} \mathbf{\tilde{Q}}^{-1} \mathbf{y} \\
    \kappa(x, x') &= c_W k_{x,x'} - (k_x \otimes \mu_W)^\top \mathbf{\tilde{C}} \mathbf{\tilde{Q}}^{-1} \mathbf{\tilde{C}}^\top(k_{x'} \otimes \mu_W)
\end{align*}  
\end{proposition}

\begin{proof}
    This immediately comes from (eq \ref{GPProxy Conditioning}) and the Kolmogorov extension theorem, same as the proof (\ref{proofIV}).
\end{proof}
\newpage

\section{Splitting the Data}
\label{appendix: split data}
Certain methods, particularly two-stage regression approaches \citep{singh2019,mastouri2021proximal,singh2023kernelmethodsunobservedconfounding}, partition the observed data into two distinct sets. For instance, in the instrumental variable setting, one might define $\mathcal{D}_1 = (\mathbf{x}, \mathbf{y}, \mathbf{z})$ and $\mathcal{D}_2 = (\tilde{\mathbf{x}}, \tilde{\mathbf{y}}, \tilde{\mathbf{z}})$. Here, $\mathcal{D}_1$ is used to train the first-stage model (estimating conditional mean embeddings), while $\mathcal{D}_2$ is reserved for the second-stage regression. For hyperparameter selection, these methods often minimize the risk associated with one stage using data from the other stage—for example, tuning second-stage hyperparameters based on performance evaluated with $\mathcal{D}_1$, and vice versa. In some practical settings, the data may be only partially available across the splits; for instance, $\mathcal{D}_1$ might contain only $(\mathbf{x}, \mathbf{z})$, while $\mathcal{D}_2$ contains only $(\tilde{\mathbf{y}}, \tilde{\mathbf{z}})$. However, none of these data-splitting scenarios are assumed in our framework. Regarding hyperparameter selection, because lengthscales typically influence both stages of the model, splitting the data and optimizing hyperparameters separately for each stage becomes cumbersome.

In our approach, the first-stage regression (i.e., the estimation of the conditional mean operator) is embedded within the computation of the GP covariance kernel. Consequently, if we were to decouple the estimation of the CMO from the overall procedure, a similar estimation strategy could be applied provided the hyperparameters are fixed. For instance, in the empirical form of the posterior mean and (co-)variance of GPIV, one could replace the mediation matrix $\mathbf{A} = (\mathbf{K}_{\mathbf{z}\mathbf{z}} + \eta I)^{-1} \mathbf{K}_{\mathbf{z}\mathbf{z}}$ with $\tilde{\mathbf{A}} = (\mathbf{K}_{\mathbf{z}\mathbf{z}} + \eta I)^{-1} \mathbf{K}_{\mathbf{z}\tilde{\mathbf{z}}}$, and replace $\mathbf{y}$ with $\tilde{\mathbf{y}}$:

\begin{equation}
\begin{aligned}
\label{IVSplit}
      \hat{\mu}(x) &= \bfK_{x\bfx}\tilde{\bfA}(\tilde{\bfA}^\top\bfK_{\bfx\bfx}\tilde{\bfA}+ \sigma^2 I)^{-1}\tilde{\bfy}, \\
    \hat{{\kappa}}(x, x') &= k_{x,x'} - \bfK_{x\bfx}\tilde{\bfA}(\tilde{\bfA}^\top\bfK_{\bfx\bfx}\tilde{\bfA}+ \sigma^2 I)^{-1}\tilde{\bfA}^\top\bfK_{\bfx x'}.
\end{aligned}   
\end{equation}

as our CMO $C_{X|Z}$ is estimated from $\mathcal{D}_1$ and the remaining GP regression comes from $\mathcal{D}_2$.

Similarly, for the proxy setting, our CMO $C_{W|X,Z}$ is estimated from $\mathcal{D}_1 = (\mathbf{x},\mathbf{y},\mathbf{z},\mathbf{w})$ and the remaining GP comes from $\mathcal{D}_2 =(\tilde{\mathbf{x}},\tilde{\mathbf{y}},\tilde{\mathbf{z}},\tilde{\mathbf{w}})$ if we choose to split our data. Therefore, the empirical posterior mean and (co-)variance of GPProxy have the form:

\begin{equation}
\begin{aligned}
\label{ProxySplit}
   \hat{\mu}(x) &= \{\mathbf{K}_{x\tilde{\mathbf{x}}} \odot (\mathbf{K}_{\bar{w}\mathbf{w}} \tilde{\mathbf{B}})\} \{\mathbf{K}_{\tilde{\mathbf{x}}\tilde{\mathbf{x}}} \odot (\tilde{\mathbf{B}}^{\top}\mathbf{K}_{\mathbf{w}\mathbf{w}} \tilde{\mathbf{B}}) + \sigma^2 I\}^{-1} \tilde{\mathbf{y}},\\
   \hat{\kappa}(x, x') &= \hat{c}_W \, k_{x,x'} - \{\mathbf{K}_{x\tilde{\mathbf{x}}} \odot (\mathbf{K}_{\bar{w}\mathbf{w}} \tilde{\mathbf{B}}) \}\{\mathbf{K}_{\tilde{\mathbf{x}}\tilde{\mathbf{x}}} \odot (\tilde{\mathbf{B}}^{\top}\mathbf{K}_{\mathbf{w}\mathbf{w}} \tilde{\mathbf{B}}) + \sigma^2 I \} ^{-1}\{\mathbf{K}_{x'\tilde{\mathbf{x}}} \odot (\mathbf{K}_{\bar{w}\mathbf{w}} \tilde{\mathbf{B}}) \}^{\top},
\end{aligned}   
\end{equation}

where we replace the mediation matrix $\mathbf{B}$ with $\tilde{\mathbf{B}}=(\bfK_{\bfx \bfx}\odot\bfK_{\bfz \bfz} + \eta I)^{-1}(\bfK_{\bfx \tilde{\bfx}}\odot\bfK_{\bfz \tilde{\bfz}})$.

However, we do not recommend doing so, especially when the data size is small. This ensures full use of the available information, which is particularly important when the dataset is limited; splitting the data in such cases would risk underutilizing the already scarce samples. From the synthetic experiments for IV (Table \ref{mseresults}), we observe that KIV methods with a training size of 1000 perform similarly to GPIV with a training size of 500. This is largely because data splitting (e.g., an equal split) reduces the effective sample size for each stage, effectively using only half of the information.

To better illustrate the effect of data splitting, we performed additional simulations in the IV settings.

\begin{table}[htbp]
\centering
\caption{MSE and Coverage: effect of data splitting (mean with standard error in parentheses).}
\label{tab:mse_coverage_split}
\adjustbox{max width=\textwidth}{
\begin{tabular}{l c c c}
\toprule
Design & $n$ & MSE & Coverage \\
\midrule
\multirow{3}{*}{log} 
 & 200 & 0.473 (0.295) & 0.917 (0.130) \\
 & 500 & 0.152 (0.076) & 0.984 (0.045) \\
 & 1000 & 0.091 (0.046) & 0.977 (0.048) \\
\addlinespace
\multirow{3}{*}{sine} 
 & 200 & 0.407 (0.149) & 0.816 (0.162) \\
 & 500 & 0.256 (0.132) & 0.852 (0.144) \\
 & 1000 & 0.144 (0.103) & 0.880 (0.202) \\
\addlinespace
\multirow{3}{*}{linear} 
 & 200 & 0.487 (0.317) & 0.934 (0.130) \\
 & 500 & 0.183 (0.098) & 0.964 (0.051) \\
 & 1000 & 0.097 (0.042) & 0.921 (0.074) \\
\addlinespace
\multirow{3}{*}{demand} 
 & 200 & 0.472 (0.028) & 0.902 (0.020) \\
 & 500 & 0.152 (0.010) & 0.945 (0.009) \\
 & 1000 & 0.062 (0.003) & 0.957 (0.008) \\
\bottomrule
\end{tabular}
}
\end{table}

Comparing our results with the GPIV column of Table \ref{mseresults} in the main text, we find that the MSE increases when we split the data, since only half of the information is leveraged. Indeed, the MSE at \(n=1000\) in the table above is almost equal to that at \(n=500\) without data splitting. Moreover, compared to KIV, our method performs slightly better on the synthetic design and substantially better on the demand design. This indicates that the hyperparameter optimization strategy yields considerable improvements in the latter case.


\newpage

\section{Equivalence Between Methods}
\subsection{Instrument Variable}
\begin{proposition}
(Equivalency between KIV, DIV, and GPIV) Assume that $\bfK_{\bfx\bfx}$, $(\bfA^\top\bfK_{\bfx\bfx}\bfA+ \sigma^2 I)$ are invertible, the estimator of KIV, DIV, and the mean posterior of GPIV are equivalent.
\end{proposition}

\begin{proof}
KIV \citep{singh2019} yields a nonparametric closed-form estimator of $f$:
\begin{equation*}
     \hat{f}_{KIV}(x) = (\hat{\alpha})^{T}\bfK_{\bfx x}
\end{equation*}
where $\hat{\alpha} =[\bfK_{\bfx\bfx}\bfA(\bfK_{\bfx\bfx}\bfA)^{T}+n\lambda \bfK_{\bfx\bfx}]^{-1}\bfK_{\bfx\bfx}\bfA\bfy$.

We then derive the following identity:
\begin{align*}
    \hat{f}_{KIV}(x) &=\bfK_{x\bfx}[\bfK_{\bfx\bfx}\bfA(\bfK_{\bfx\bfx}\bfA)^{\top}+n\lambda\bfK_{\bfx\bfx}]^{-1}\bfK_{\bfx\bfx}\bfA\bfy \\
    & = \bfK_{x\bfx}[\bfK_{\bfx\bfx}\bfA\,\bfA^{\top}\bfK_{\bfx\bfx}+n\lambda\bfK_{\bfx\bfx}]^{-1}\bfK_{\bfx\bfx}\bfA\bfy \\
    & = \bfK_{x\bfx}[\bfK_{\bfx\bfx}(\bfA\,\bfA^{\top}\bfK_{\bfx\bfx}+n\lambda I)]^{-1}\bfK_{\bfx\bfx}\bfA\bfy\\
    & = \bfK_{x\bfx}(\bfA\,\bfA^{\top}\bfK_{\bfx\bfx}+n\lambda I)^{-1}\, \bfK_{\bfx\bfx}^{-1}\,\bfK_{\bfx\bfx}\bfA\bfy\\
    & = \bfK_{x\bfx}(\bfA\,\bfA^{\top}\bfK_{\bfx\bfx}+n\lambda I)^{-1}\,\bfA\bfy\\
    & = \bfK_{x\bfx}\,\bfA\,(\bfA^{\top}\bfK_{\bfx\bfx}\,\bfA+n\lambda I)^{-1}\bfy\\
    & = \hat{\mu}(x)
\end{align*}

where the last step applies the push-through identity—an important corollary of the Woodbury identity.

This closed-form expression coincides with that of Deconditioned IV and with the posterior mean of GPIV, provided the noise variance $\sigma^2$ is interpreted as the inverse regularization parameter $n \cdot \lambda$ in the second stage.

Now we only need to prove the equivalency between DIV and GPIV. This is trivial as the estimator in DIV (eq \ref{DMOIV}) and GPIV (eq \ref{GPIV}) are identical.
\end{proof}

\begin{remark}
The equivalence remains valid when the data are split according to the scheme described in Appendix \ref{appendix: split data} for both KIV and GPIV, under the assumption that $\bfK_{\bfx\bfx}$ and $(\tilde{\bfA}^\top\bfK_{\bfx\bfx}\tilde{\bfA}+ \sigma^2 I)$ are invertible.
\end{remark}

Therefore, the posterior mean of GPIV inherits the good asymptotic behaviors of KIV \citep{singh2019,lob2026uniforminferencekernelinstrumental}.

\subsection{Proximal Causal Learning}
\label{appendix: ProxyEqv}

\textbf{GPProxy and KNC}
\vspace{0.2cm}
\begin{proposition}
Assuming $\mathbf{K}_{\mathbf{x}\mathbf{x}} \odot (\mathbf{B}^{\top}\mathbf{K}_{\mathbf{w}\mathbf{w}} \mathbf{B})$ is invertible, the posterior mean of GPProxy is equivalent to the estimator given by KNC \citep{singh2023kernelmethodsunobservedconfounding}.
\end{proposition}

\begin{proof}
KNC gives a closed-form estimator of bridge function:
$$
\hat{h}(x,w) = \{\mathbf{K}_{x\mathbf{x}} \odot (\mathbf{K}_{\bar{w}\mathbf{w}} \mathbf{B})\}\hat{\mathbf{\alpha}},
$$
where $\hat{\mathbf{\alpha}}=(\bfM\bfM^{T}+n\lambda\bfM)^{-1}\bfM\bfy$, $\bfM = \mathbf{K}_{\mathbf{x}\mathbf{x}} \odot (\mathbf{B}^{\top}\mathbf{K}_{\mathbf{w}\mathbf{w}} \mathbf{B}) = \bfM^{T}$, and $\bfB = (\bfK_{\bfx \bfx}\odot\bfK_{\bfz \bfz} + \eta I)^{-1}(\bfK_{\bfx \bfx}\odot\bfK_{\bfz \bfz})$ is mediation matrix.

Now we have:
\begin{align*}
 \hat{\mathbf{\alpha}}&=(\bfM\bfM^{T}+n\lambda\bfM)^{-1}\bfM\bfy\\
&= \{\bfM(\bfM^{T}+n\lambda I)\}^{-1}\bfM \bfy\\
&=(\bfM^{T}+n\lambda I)^{-1}\bfM^{-1}\bfM\bfy\\
&=(\bfM^{T}+n\lambda I)^{-1}\bfy\\
&=\{\mathbf{K}_{\mathbf{x}\mathbf{x}} \odot (\mathbf{B}^{\top}\mathbf{K}_{\mathbf{w}\mathbf{w}} \mathbf{B})+n\lambda I\}^{-1}\bfy.
\end{align*}
Moreover, 
\begin{align*}
\hat{f}_{KNC}(x) &= \frac{1}{n}\sum_{i=1}^{n}\hat{h}(x,w_i)\\
&= \frac{1}{n}\sum_{i=1}^{n}(\bfK_{x\bfx}\odot\bfK_{w_i\bfw}\bfB)\hat{\mathbf{\alpha}}\\
&= (\bfK_{\bfx\bfx} \odot \bfK_{\bar{w}\bfw}\bfB)\hat{\mathbf{\alpha}},
\end{align*}
which equals to the mean estimator of GPProxy - the noise variance $\sigma^2$ is interpreted as the inverse regularization hyperparameter $n\cdot\lambda$ in the second stage regression.
\end{proof}

\begin{remark}
The equivalence remains valid when the data are split according to the scheme described in Appendix \ref{appendix: split data} for both KNC and GPProxy, assuming $\mathbf{K}_{\mathbf{\tilde{x}}\mathbf{\tilde{x}}} \odot (\tilde{\mathbf{B}}^{\top}\mathbf{K}_{\mathbf{w}\mathbf{w}} \tilde{\mathbf{B}})$ is invertible.
\end{remark}

Similarly, the posterior mean of GPProxy inherits the asymptotic behaviours from KNC \citep{singh2023kernelmethodsunobservedconfounding}.

\textbf{KNC and KPV}

Both KNC and KPV aim to minimize the empirical loss:
\begin{align*}
&\hat{\eta}_{XW} = \text{argmin}_{\eta \in \mathcal{H_{\mathcal{X}}} \times \mathcal{H_{\mathcal{W}}}}\hat{L}(\eta), \,\text{where}\\
&\hat{L}(\eta) = \frac{1}{n}\sum_{i=1}^{n}(y_i - \langle\eta,k_x\otimes \hat{\mu}_{W|x,z}\rangle_{\mathcal{H_{\mathcal{X}}} \times \mathcal{H_{\mathcal{W}}}})+\lambda||\eta||_{{\mathcal{H_{\mathcal{X}}} \times \mathcal{H_{\mathcal{W}}}}}.
\end{align*}
Then, $\hat{f}(x) = \langle \hat{\eta}_{XW},\hat{\mu}_{W}\otimes k_x\rangle_{\mathcal{H_{\mathcal{X}}} \times \mathcal{H_{\mathcal{W}}}}$. Both methods use the representer theorem to express $\hat{\eta}_{XW}$ as a linear combination of tensor products of kernel features and CME. However, in KNC, the empirical CME $\hat{\mu}_{W|x,z}$ is first estimated from the first-stage regression, and then $\hat{\eta}_{XW}$ is represented as $\hat{\eta}_{XW} = \sum_{i=1}^{n} \alpha_i k_{x_i} \otimes \hat{\mu}_{W|x_i,z_i}$, where the coefficients $\alpha_i$ form a vector. In contrast, KPV first expresses $\hat{\mu}_{W|x_i,z_i} = \sum_{j=1}^{n} \Gamma_{j}(x_j,z_j)k_{w_j}$, and then represents $\hat{\eta}_{XW}$ as $\hat{\eta}_{XW} = \sum_{i=1}^{n} \sum_{j=1}^{n}\alpha_{ij} \Gamma_{j}(x_j,z_j) k_{x_i} \otimes k_{w_j}$, leading to a matrix of parameters $\alpha_{ij}$, which differs from KNC's vector of parameters $\alpha_i$. This leads to the different closed-form estimator between KPV and KNC (and the posterior mean of GPProxy).

\textbf{DProxy and KNC-orig}

Now, we introduce an alternative kernel regression for the proxy settings, which we named as KNC-orig in our paper. KNC-orig adopts the two stage regression as well, and in the first stage, the operator $C_{X,W|X,Z}$ is obtained from kernel regression. In the second stage, KNC-orig learns the $\hat{f}$ as a mapping from $y$ to $\mu_{X,W|X,Z}$. KPV \citep{mastouri2021proximal} and KNC \citep{singh2023kernelmethodsunobservedconfounding}, on the other hands, first obtain the operator $C_{W|X,Z}$ and then regress $y$ against $(\mu_{W|x,z},k_x)$. \citet{mastouri2021proximal} points out that the first stage of KNC-orig contains self-regression $k_x$ to $k_x$, which violates the completeness of assumptions, and hence causes bias. Similarly, since $D_{X,W|X,Z}$ contains the CMO $C_{X,W|X,Z}$, it violates the completeness of the assumptions as well.

KNC-orig gives a closed form of the estimator:
\begin{equation}
\begin{aligned}
   \label{KNC-orig}
  &\hat{f}_{KNC-o}(x) = \frac{1}{n}\sum_{i=1}^{n}(\bfK_{x\bfx}\odot\bfK_{w_i\bfw})[(\bfK_{\bfx\bfx}\odot\bfK_{\bfw\bfw})\bfB\bfB^\top (\bfK_{\bfx\bfx}\odot\bfK_{\bfw\bfw})+n\lambda (\bfK_{\bfx\bfx}\odot\bfK_{\bfw\bfw})]^{-1}
  \\&(\bfK_{\bfx\bfx}\odot\bfK_{\bfw\bfw})\bfB\bfy,
\end{aligned}
\end{equation}

where $\bfB$ is the mediation matrix.

\begin{corollary}
Under mild invertible assumptions, DProxy and KNC-orig give the same estimator of $f$.
\end{corollary}

\begin{proof}
Again, using the Woodbridge identity, we have
\begin{align*}
   & [(\bfK_{\bfx\bfx}\odot\bfK_{\bfw\bfw})\bfB\bfB^\top (\bfK_{\bfx\bfx}\odot\bfK_{\bfw\bfw})+n\lambda(\bfK_{\bfx\bfx}\odot\bfK_{\bfw\bfw})]^{-1}(\bfK_{\bfx\bfx}\odot\bfK_{\bfw\bfw})\bfB \\
    &= [\bfB\bfB^\top (\bfK_{\bfx\bfx}\odot\bfK_{\bfw\bfw})+n\lambda I]^{-1}\bfB\\
    &= \bfB[\bfB^\top (\bfK_{\bfx\bfx}\odot\bfK_{\bfw\bfw})\bfB+n\lambda I]^{-1}
\end{align*}

Rewriting the summation in equation (eq \ref{KNC-orig}): 
$$
\frac{1}{n}\sum_{i=1}^{n}(\bfK_{x\bfx}\odot\bfK_{w_i\bfw}) = (\bfK_{x\bfx}\odot\bfK_{\Bar{w}\bfw}),
$$
we have 
\begin{equation*}
  \hat{f}_{KNC-o}(x) =  (\bfK_{x\bfx}\odot\bfK_{\Bar{w}\bfw})\bfB[\bfB^\top (\bfK_{\bfx\bfx}\odot\bfK_{\bfw\bfw})\bfB+n\lambda I]^{-1}\bfy,
\end{equation*}
which is in the same form of DProxy estimator (eq \ref{DProxy}).  
\end{proof}

\newpage

\section{First stage Uncertainty Quantification}
\label{UQ1}
To incorporate the first-stage estimation uncertainty into the whole pipeline, we use bootstrap procedures to obtain an empirical distribution of the CMOs of the first stage, and subsequently incorporate them into the second stage one by one, and obtain a hierarchical distribution over the Gaussian processes.
We can now analyse the uncertainty induced by adding this extra procedure. We decompose the total uncertainty in this case. We let $f \mid \mathcal{D}$ be denoted by $f$, and let $C$ be the estimator of $C_{X|Z}$ (the conditional Mean operator). Then
$$
\mathrm{Var}[f] = \mathrm{Var}_C\bigl[\mathbb{E}_f[f\mid C]\bigr] + \mathbb{E}_C\bigl[\mathrm{Var}_f[f\mid C]\bigr].
$$

The second term is the second‑stage uncertainty; our posterior variance provides an asymptotically unbiased point estimate of it as the empirical estimator of conditional mean embedding converges to its population counterpart in the RKHS norm with rate $O(m^{-1/4})$, with the decrease of regularization \citep{song2012feature}. The first term captures first‑stage variance, which can be estimated via bootstrap on $\hat{C}_{X|Z}$ or a model‑based approach (e.g., a Bayesian prior for the CMO). The former comes from the uncertainty due to having a finite sample in the first stage, i.e., the sampling uncertainty, whereas the latter comes from the prior of the model. We focus on the former here and will explore the latter in our follow‑up work.

In practice, we bootstrapped the data for the first stage $\mathcal{D}_{1}$ and estimated different CMOs, then used them to obtain the second stage mean estimator from the second stage data $\mathcal{D}_{2}$ which is not being bootstrapped. If we do not split data then $\mathcal{D}_{1}= \mathcal{D}_{2}= \mathcal{D} $. We will obtain multiple posterior means for $f$. Estimating their variance would give us a bootstrap estimate of the first term in the variance decomposition. For the second stage, it would be the posterior variance from the original data.

\begin{figure}[htbp]  
  \centering
  \includegraphics[width=0.7\columnwidth]{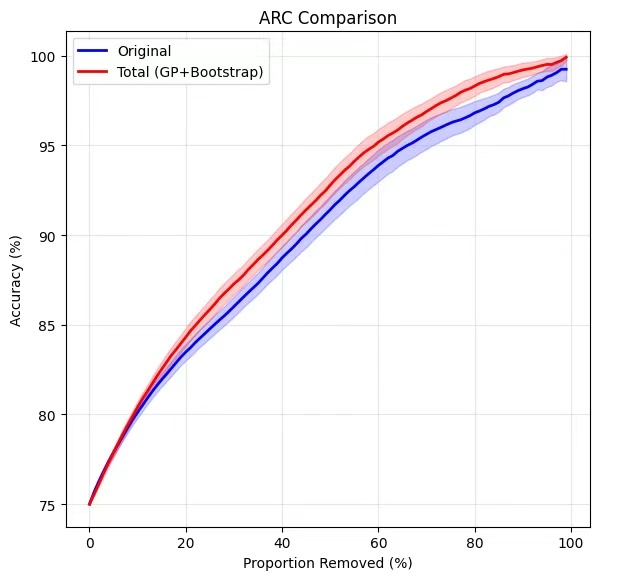}
  \caption{ARC curves for the Demand design: tiny improvement from counting first stage uncertainty.}
  \label{ARCDemandUQ1}
\end{figure}

\begin{figure}[htbp]  
  \centering
  \includegraphics[width=0.7\columnwidth]{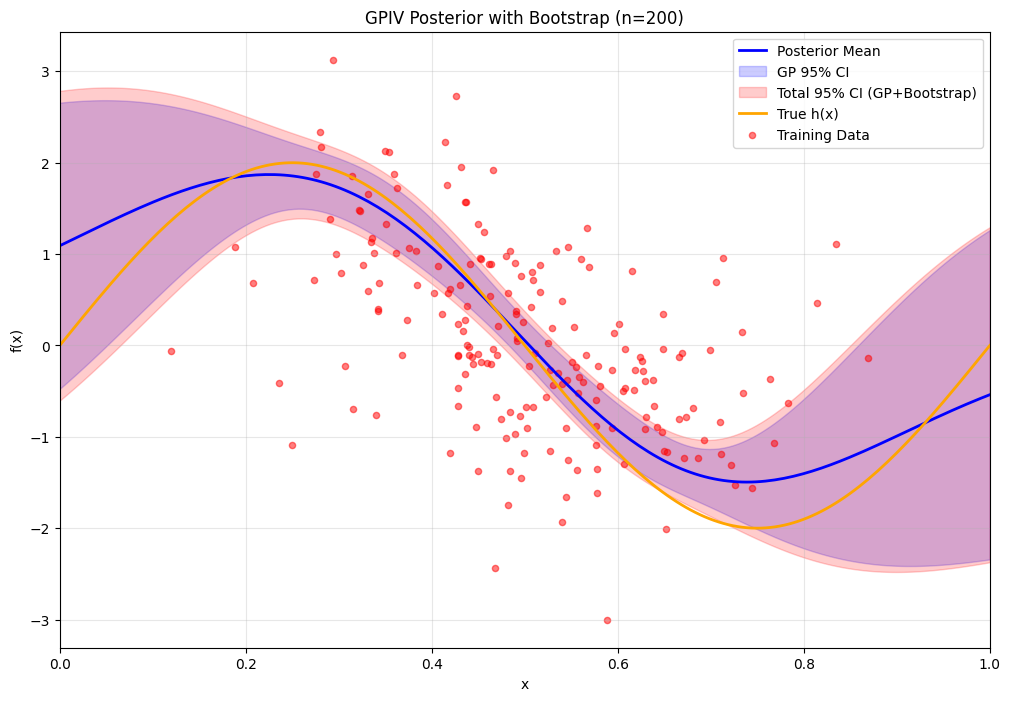}
  \caption{Demo of the first stage UQ in the sine design.}
  \label{DemoUQ1Sine}
\end{figure}

Experiments (Table \ref{tab:UQ1}, Figs \ref{ARCDemandUQ1}-\ref{DemoUQ1Sine}) show that the second‑stage variance dominates the first‑stage variance, often by a factor of $10$ to $1000$. More interestingly, the first‑stage variance provides a small positive “correction” to UQ: when coverage is below $95\%$, it increases coverage by $3‑10\%$; otherwise, the improvement is negligible $(0 - 0.5\%)$. In terms of ARC, the first‑stage variance makes no substantial difference. Given that the uncertainty largely originates from the second‑stage prior on $f$, and that incorporating first‑stage uncertainty would break the closed‑form posterior variance of the second stage, we therefore do not incorporate this first‑stage UQ into the main paper.

\begin{table}[htbp]
\centering
\caption{AUC and Coverage for Second Stage Uncertainty and Combined Uncertainty (mean with standard error in parentheses).}
\label{tab:UQ1}
\adjustbox{max width=\textwidth}{
\begin{tabular}{l l c c c c}
\toprule
Design & IV Strength & \textbf{Second Stage cov} & \textbf{Combined cov} & \textbf{Second Stage AUC} & \textbf{Combined AUC} \\
\midrule
\multirow{3}{*}{log} 
 & high   & 0.8992(0.0929) & 0.9026(0.0905) & 0.8921(0.0285) & 0.8984(0.0298) \\
 & median & 0.8886(0.0813) & 0.8962(0.0795) & 0.9137(0.0033) & 0.9125(0.0035) \\
 & low    & 0.7618(0.2736) & 0.7986(0.2512) & 0.9102(0.0047) & 0.9093(0.0102) \\
\addlinespace
\multirow{3}{*}{sine} 
 & high   & 0.8640(0.1206) & 0.8724(0.1173) & 0.8875(0.0254) & 0.8903(0.0260) \\
 & median & 0.8388(0.1064) & 0.9392(0.0774) & 0.8318(0.0768) & 0.8435(0.0855) \\
 & low    & 0.8034(0.1731) & 0.8948(0.1307) & 0.7998(0.0758) & 0.7953(0.0799) \\
\addlinespace
\multirow{3}{*}{linear} 
 & high   & 0.8764(0.1088) & 0.8784(0.1173) & 0.8894(0.0192) & 0.8993(0.0244) \\
 & median & 0.9146(0.1424) & 0.9170(0.1416) & 0.9025(0.0239) & 0.9017(0.0219) \\
 & low    & 0.9542(0.0759) & 0.9596(0.0710) & 0.9123(0.0041) & 0.9117(0.0049) \\
\addlinespace
demand  & - & 0.8590(0.0466) & 0.9160(0.0486) & 0.8698(0.0342) & 0.8878(0.0385) \\
\bottomrule
\end{tabular}
}
\end{table}

\newpage

\section{Extension of GPProxy with observed confounders}
Some of existing literature \citep{mastouri2021proximal,singh2023kernelmethodsunobservedconfounding} also considers scenarios where a subset of the confounders, denoted by $C$, is observed, as represented in the DAG \ref{DAGProxyEx}. Under analogous identification assumptions (for details, see \citet{mastouri2021proximal}), the average treatment effect (ATE) remains identifiable via a Fredholm integral equation of the form:

\begin{equation*}
    \EE[Y\mid X=x,C=c,Z=z] = \int h(x,c,W) dP(W\mid X=x,C=c,Z=z).
\end{equation*}
Marginalizing $C$ and $W$ out from the bridge function $h$, we have the ATE:
\begin{equation*}
    f(x) = \EE[Y \mid do(X=x)] = \int h(x,c,w) dP(C,W).
\end{equation*}

\begin{figure}[htbp]
    \centering
    \includegraphics[width=0.45\textwidth]{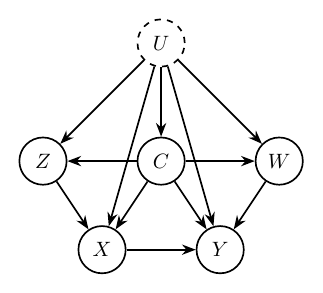}
    \caption{DAG of Proxy setting when the observed confounders exist.}
    \label{DAGProxyEx}
\end{figure}

The procedure of construction for the posterior $f$ is similar to the Section 3.2:
\begin{itemize}
    \item Elicit a GP prior for the bridge function: $h \sim \cG\cP(0, k_\cX\cdot k_\cC \cdot k_\cW)$.
    \item Let $g(x,c,z) = \EE[Y\mid X=x,C=c,Z=z]$, under regularity assumptions on the covariance kernel, $g \sim \cG\cP(0,q)$ is still a GP under the integration which is a linear operation. Then covariance kernel $q((x,c,z),(x',c',z')) = k_{x,x'}k_{c,c'}\langle\mu_{W\mid x,c,z},\mu_{W\mid x’,c',z'}\rangle_{\cH_\cW}$.
    \item Similarly, $f(x) \sim \cG\cP(0,r)$ under the regularity assumptions, where $r(x,x') = k_{x,x'}\lVert \mu_{W} \rVert^{2}_{\mathcal{H}_{\mathcal{W}}}$.
    \item Consider the additive noise model: $\bfy \mid \bfx,\bfc,\bfz \sim N(g(\bfx,\bfc,\bfz),\sigma^2 I)$. The properties above guarantees the joint Gaussianity between $\bfy$ and $f(\bfx)$, where 
    $$\mathrm{cov}(y,f(x)) = k_{x,x'}\langle\mu_{C},k_{c'}\rangle_{\cH_\cC}\langle\mu_{W},\mu_{W'\mid c',x',z'}\rangle_{\cH_\cW}$$.
    \item From Gaussian conditioning, the posterior of $f$ is still a GP with mean $\mu(x)$ and covariance kernel $\kappa(x,x')$:
\begin{equation*}
    \begin{aligned}
        \mu(x) &= (k_x \otimes \mu_C \otimes \mu_W)^\top (\Psi_\mathbf{x} \otimes \Psi_\bfc \otimes C_{W\mid C,X,Z}\Psi_{\bfc, \mathbf{x},\mathbf{z}} )(\mathbf{Q} + \sigma^2I)^{-1} \mathbf{y}   \\
        \kappa(x, x') &= k_{x,x'}\,c_{CW} - (k_x \otimes \mu_C \otimes \mu_W)^\top (\Psi_\mathbf{x} \otimes \Psi_\bfc\otimes C_{W\mid C,X,Z}\Psi_{\bfc,\mathbf{x},\mathbf{z}})(\mathbf{Q} + \sigma^2I)^{-1} \\
        &\quad (\Psi_\mathbf{x}\otimes \Psi_\bfc \otimes C_{W\mid C, X,Z}\Psi_{\bfc,\mathbf{x},\mathbf{z}})^{\top}(k_{x'} \otimes \mu_C  \otimes \mu_W)
    \end{aligned}
\end{equation*}
\end{itemize}

where 
\begin{align*}
\mu_W &= \int k_W \, dP(W),\\
\mu_C &= \int k_C \, dP(C),\\
c_{CW} &= \iint \langle \mu_C \otimes \mu_W, \mu_{C'} \otimes \mu_{W'} \rangle_{\cH_\cC \otimes \cH_\cW} dP(C,W)dP(C',W').
\end{align*}

Similarly, $\mathbf{Q} = (C_{W\mid C,X,Z}\Psi_{\bfc,\mathbf{x},\mathbf{z}})^{\top}(C_{W\mid C,X,Z}\Psi_{\bfc,\mathbf{x},\mathbf{z}})$.

Substituting the kernel mean embeddings $\mu_W$, $\mu_C$ and the operator $C_{W\mid C,X,Z}$ with their empirical estimators yields the corresponding estimators $\hat{\mu}(x)$ and $\hat{\kappa}(x,x')$.



\newpage

\section{Experiments Details and Additional Simulations}
\subsection{RBF kernel}
We use an RBF kernel with amplitude as our GP prior in all of our experiment designs:
\begin{equation*}
    k(x_1,x_2) = \exp\left(-\frac{(x_1-x_2)^2}{2\,l^2_{x}}\right),
\end{equation*}
where $l_x$ represents the lengthscale.

\subsection{IV Experiments Settings}
\textbf{Synthetic data.}\quad 
Here are the details of data generating process:

\begin{align*}
    \begin{bmatrix}
    e\\
    V\\
    W
    \end{bmatrix}
    \sim 
    \cN\Bigg(
    \begin{bmatrix}
    0 \\
    0 \\
    0
    \end{bmatrix},
    \begin{bmatrix}
    1& \rho &0\\
   \rho & 1 & 0\\
     0&0&1
    \end{bmatrix}
    \Bigg)
\end{align*}
and $$X = \Phi\left(\alpha W+(1-\alpha)V\right),\quad Z = \Phi(W).$$ 

In the main studies, we set the power of confounding $\rho = 0.5$ and IV strength to be moderate ($\alpha = 0.5$). Moreover, observations are generated by $Y = f(X) + e, E[e|Z] = 0$. Therefore, we have $f(x) = E[Y \mid do(X=x)]$ which is the average treatment effect (ATE). We test three functions $f(x)$:
\begin{enumerate}
    \item Sine: $f(x) = 2\sin(2\pi x)$ (Same as DualIV's design)
    \item Log (Sigmoid): $f(x) = \log (|16x-8|+1)\times\mathrm{sgn}(x-0.5)$ (Same as KIV's design)
    \item Linear: $f(x) = 4x-2$ (Same as KIV's design)
\end{enumerate}

We visualize examples for the three designs in figure \ref{syndemo}.

\begin{figure}[b]
    \centering
    \includegraphics[width=0.32\textwidth]{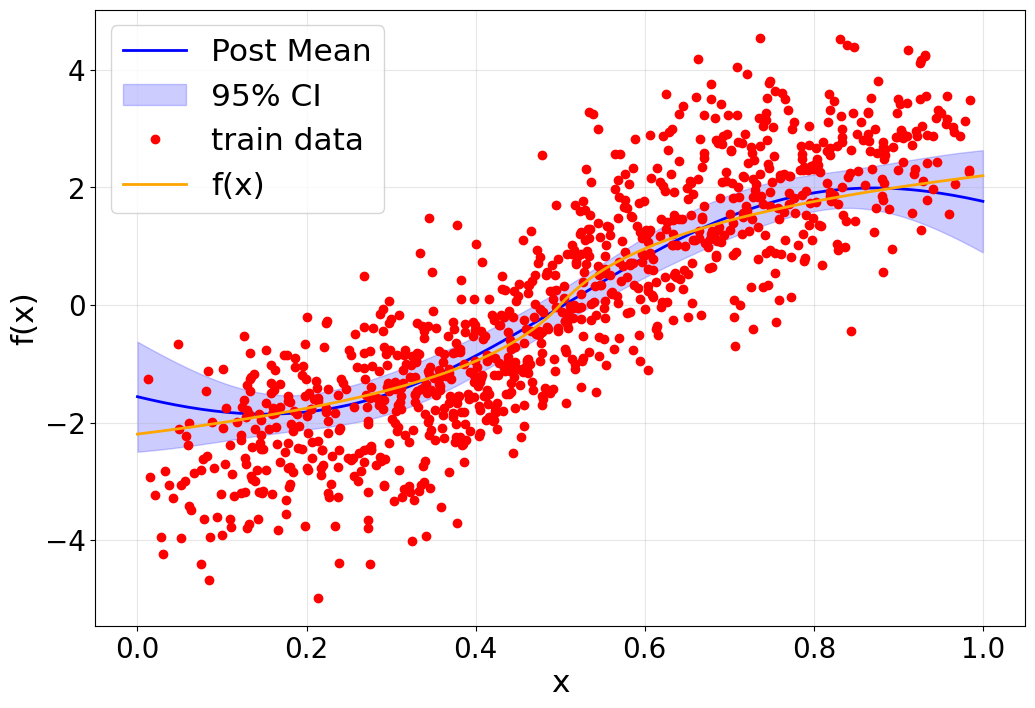}
    \hfill
    \includegraphics[width=0.32\textwidth]{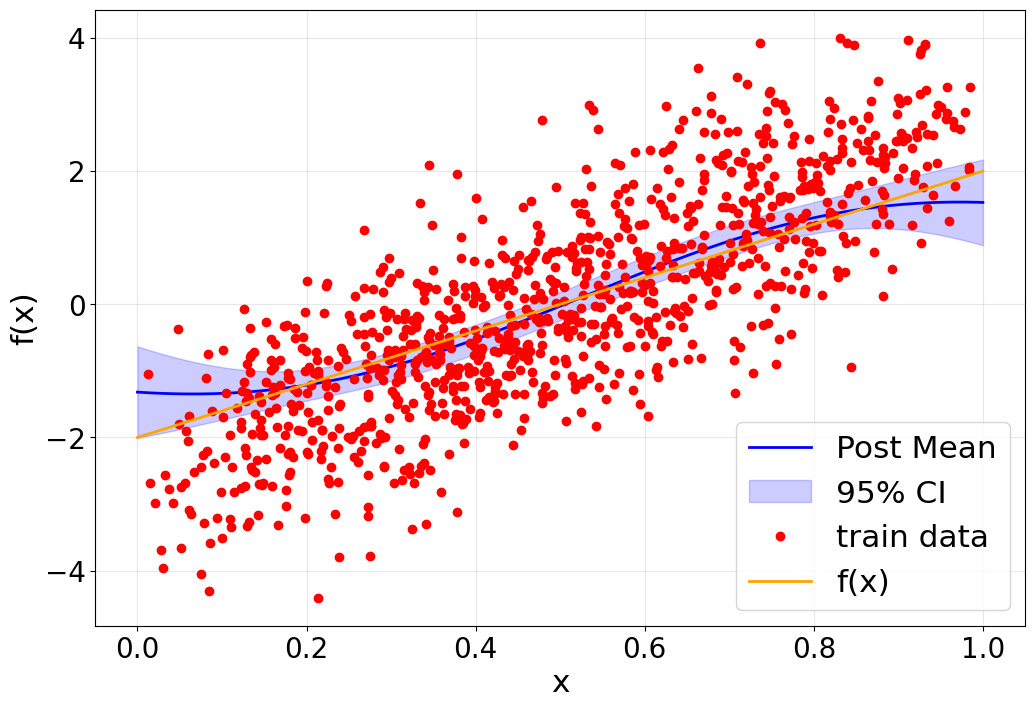}
    \hfill
    \includegraphics[width=0.32\textwidth]{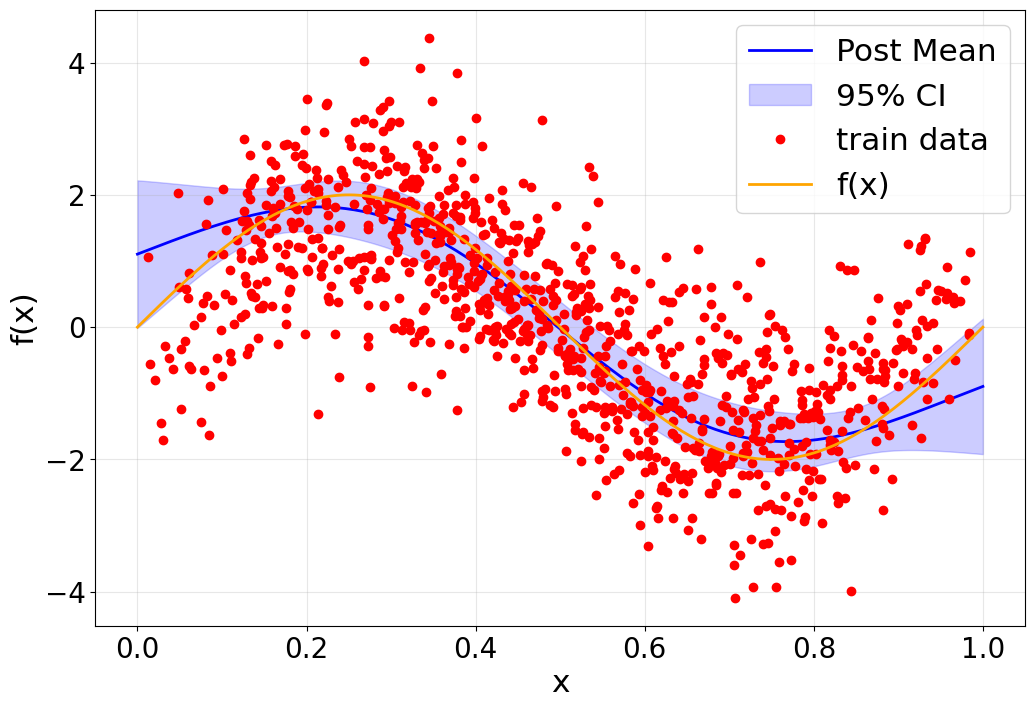}
    \caption{Demo of synthetic data with data size $= 1000$. Left: \text{log}, middle: \text{linear}, right: \text{sine}.}
    \label{syndemo}
\end{figure}

We provide the ARC for \text{sine} and \text{linear} design omitted in the main pages, which shows same behavior across three designs that GPIV perform best across other competitors. Compared with the previous version of this draft, we unified the seed across different functions and evaluations in the synthetic data. The updated result in the table \ref{mseresults} shows small but positive changes and the claim and interpretation of the results do not change. 

\begin{figure}[]
    \centering
    \includegraphics[width=0.95\textwidth]{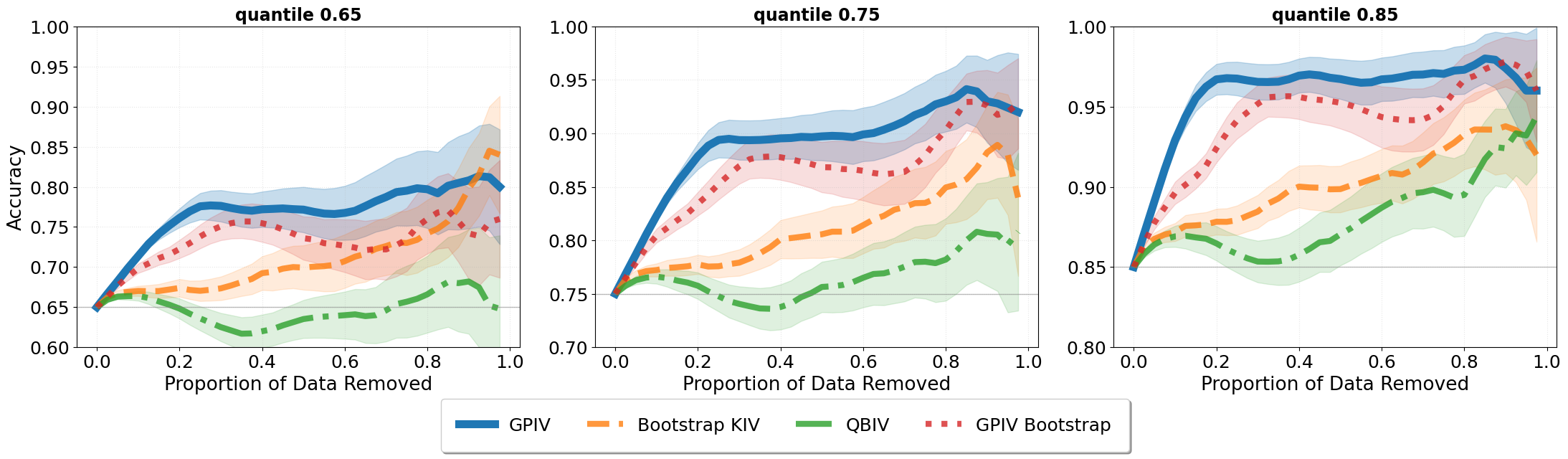}
    \caption{Accuracy-rejection curve for sine design (data size $=200$), with quantile $= 0.65,0.75,0.85$.}
    \label{arcsine}
\end{figure}
\begin{figure}[]
    \centering
    \includegraphics[width=0.95\textwidth]{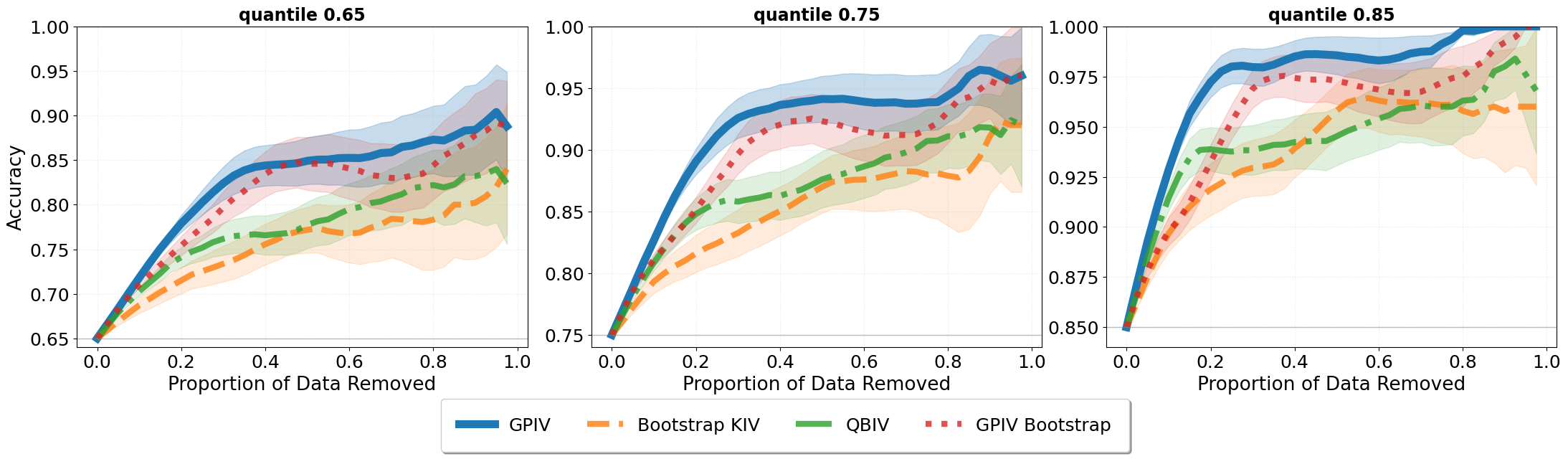}
    \caption{Accuracy-rejection curve for linear design (data size $=200$), with quantile $= 0.65,0.75,0.85$.}
    \label{arclinear}
\end{figure}
\begin{figure}[]
    \centering
    \includegraphics[width=0.95\textwidth]{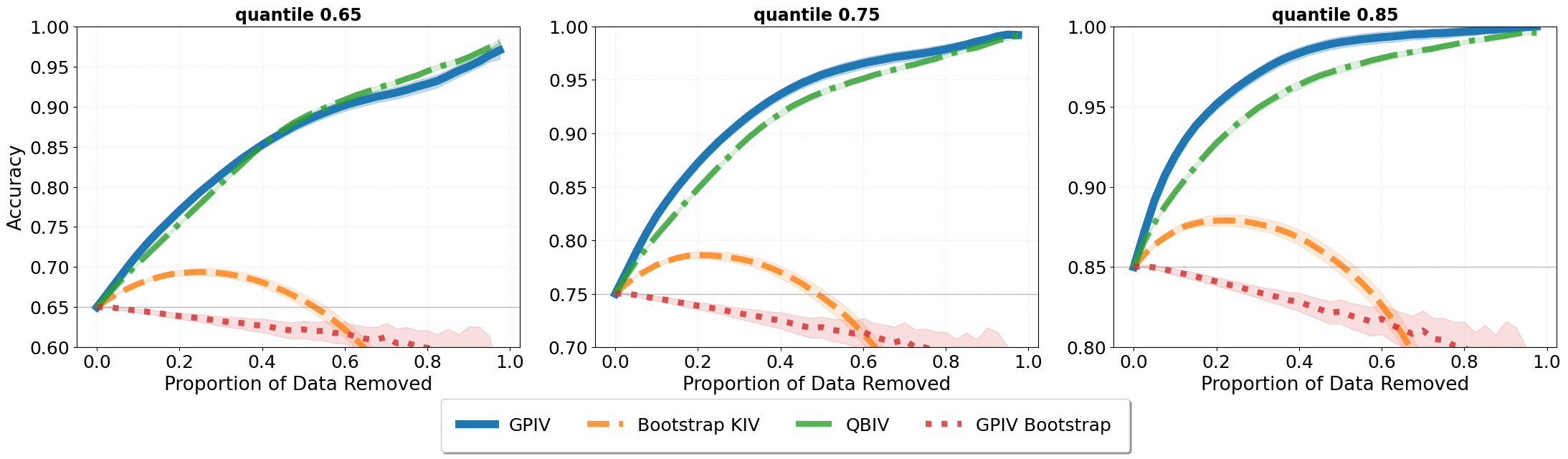}
    \caption{Accuracy-rejection curve for demand design (data size $=200$), with quantile $= 0.65,0.75,0.85$.}
    \label{arcdemand}
\end{figure}

\begin{figure}[]
    \centering
    \includegraphics[width=0.45\textwidth]{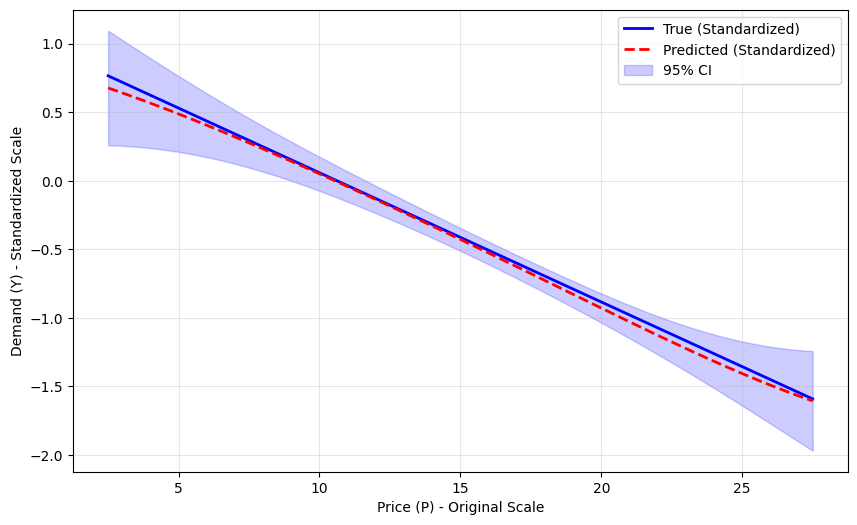}
    \hfill
    \includegraphics[width=0.45\textwidth]{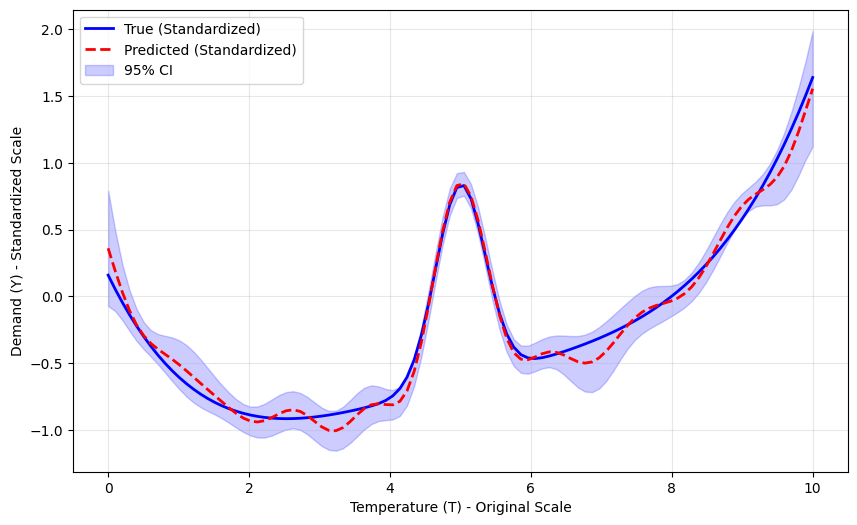}
    \caption{Demo of demand data with data size $=1000$. Left: the estimated $\hat{f}(p,t=2 ,s = 4)$. Right: the estimated $\hat{f}(p=20,t,s=4)$.}
    \label{demanddemo}
\end{figure}

\textbf{Demand data.}\quad
Data generating process:
\begin{align*}
    &S \sim U\{1,2,3,4,5,6,7\}\\
    &T \sim U[0,10]\\
    &C \sim N(0,1)\\
    &V \sim N(0,1)\\
    &\epsilon \sim N(\rho V, 1-\rho^2)\\
    & P = 25+(C+3)h(T)+V
\end{align*}

The observation are generated by $Y = f(P,T,S) + \epsilon, \mathbb E[\epsilon|C,T,S] = 0$. Here we shall consider $X=(P,T,S)$ as the input, and $Z=(C,T,S)$ as the instruments, as the KIV paper suggested.

The structural function is designed as follow:
\begin{align}
    &f(P,T,S) = 100+S(10+P)f(T) - 2P \\
    &f(T) = 2\bigg( \frac{(T-5)^4}{600}+\exp(-4(T-5)^2)+\frac{T}{10} -2 \bigg)
\end{align}

Here we are interested in ATE $f(P,T,S)$, notably, in the DFIV paper, they consider $T,S$ as covariates, $P$ and $C$ as input and instrument variables. Therefore, in their paper, ATE is considered as $f(p) = \mathbb E[Y|do(P=p)] = \mathbb E_{T,S}[f(p,T,S)]$, and CATE is defined as $\mathbb E[Y|do(P=p),T=t] = \mathbb E_{S}[f(p,t,S)]$.

\subsection{Ablation Studies on Standard GP and GPIV}
When confounding is present, both the posterior mean and variance from a standard Gaussian Process (GP) are affected. Specifically, the posterior mean becomes biased, and the posterior variance is smaller than that of GPIV. This results in an over-confident $95\%$ confidence interval and consequently poor coverage.

To illustrate this point, we amplify the confounding effect by increasing the correlation coefficient $\rho$ to $0.9$ in the synthetic data generation process.

\begin{figure}[]
    \centering
    \includegraphics[width=0.85\textwidth]{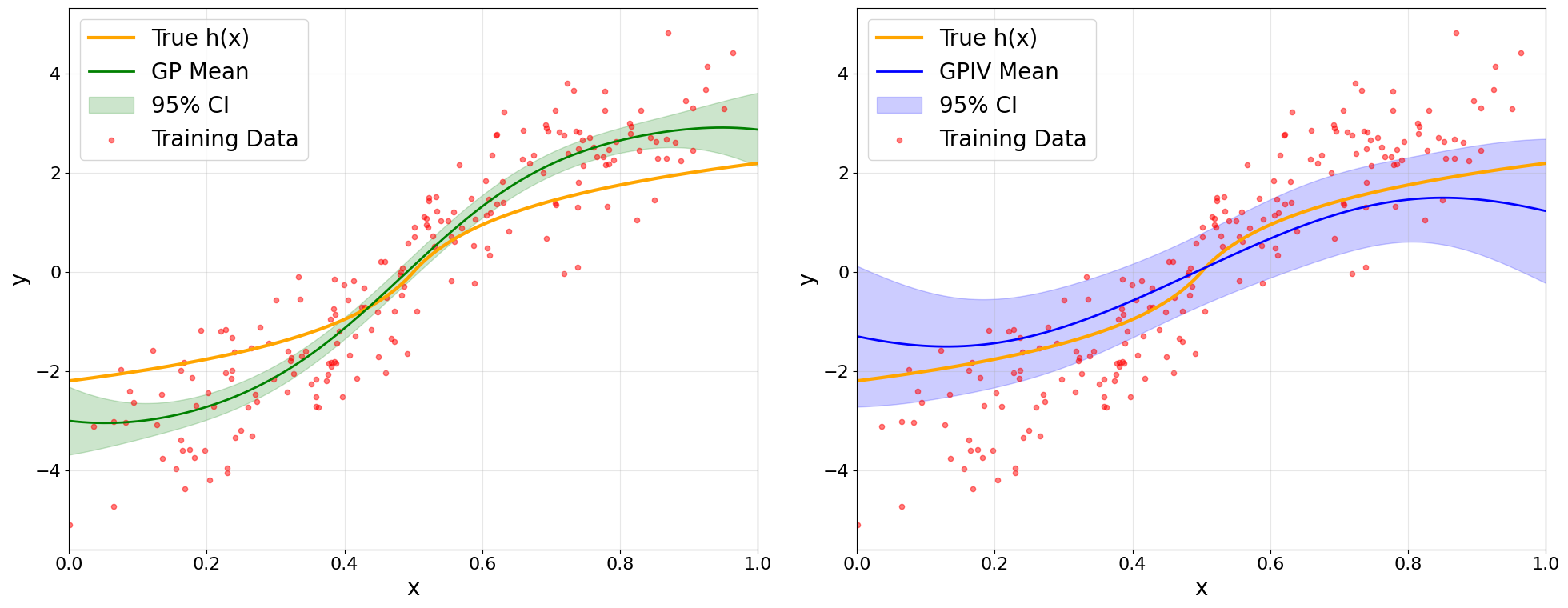}
    \caption{Demo Plots of log design of Standard GP and GPIV with data size $=200$}
    \label{StandardGPlog}
\end{figure}

Figure \ref{StandardGPlog} exemplifies the comparison between the standard GP and GPIV, illustrating that the posterior mean of the standard GP tends to be biased and its posterior variance is underestimated. This leads to an inappropriately narrow confidence interval and consequently inadequate coverage performance. 
\begin{figure}[]
    \centering
    \includegraphics[width=0.75\textwidth]{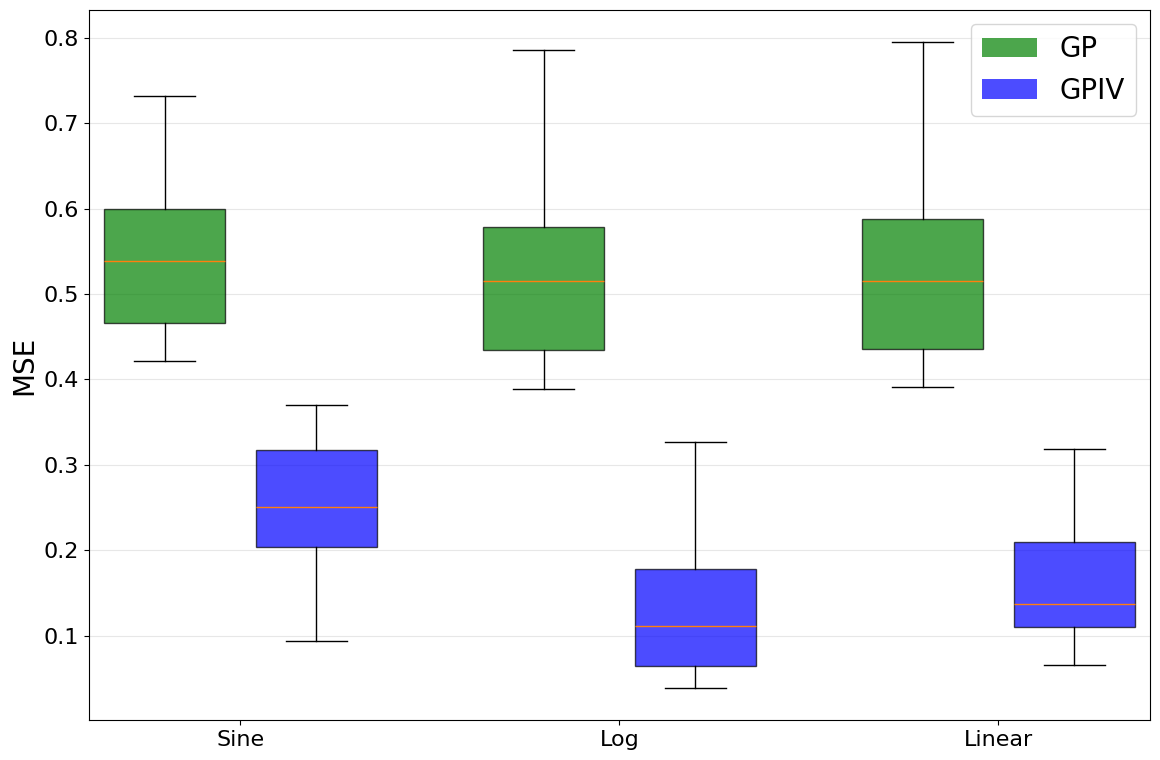}
    \caption{MSE comparison of Standard GP and GPIV.}
    \label{MSEStandardGP}
\end{figure}

Figure \ref{MSEStandardGP} indicates that bias is consistently observed in the standard GP across all experimental designs due to confounding. Furthermore, the coverage performance of the standard GP is severely compromised, with empirical coverage rates falling between only $13\%$ and $25\%$ across the different designs.

\subsection{Additional Studies on Active Learning}

\label{appendix:activelearning}
We provide an addition study on active learning to test the usefulness and quality of our uncertainty. Active learning \citep{aggarwal2014active,gal2017deep} is an approach where the model actively selects the most informative data points to be labeled by. Instead of learning from a random, pre-labeled dataset, it asks questions or requests labels for the specific instances it finds most uncertain or valuable. This interactive process allows the model to achieve high accuracy with significantly less labeled data. In causal inference and clinical trial design, labeling patient data or conducting new experiments is often costly. \citet{gao2025activecq} address this by proposing ActiveCQ, an active learning framework for efficiently estimating causal quantities. We test this on our IV setting. Compare with \citet{gao2025activecq}, we rank the pooled data by posterior variance and choose the data with the highest posterior variance instead of designed acquisition function:

\begin{itemize}
    \item Initial Training Data: $\mathcal{D}_{T} = (x,y,z)_{1:n_T}$, Pooled Data: $\mathcal{D}_{P} = (\tilde{x},\tilde{y},\tilde{z})_{1:n_P}$, and Test data.)
    \item rank all the pooled data $d \in \mathcal{D}_P$ by their posterior variance, and pick the one with the highest posterior variance, move it to the $\mathcal{D}_T$. Record the MSE.
    \item Repeat the previous step until reach the budgets. Plot the MSE against the number of acquired data points.
\end{itemize}

We set the number of initial training data to be $30$, and size of pooled data to be $300$. Each time we add one data with the highest variance to the training data, and the budget is $100$.

\begin{figure}[t]
    \centering
    \includegraphics[width=0.475\textwidth]{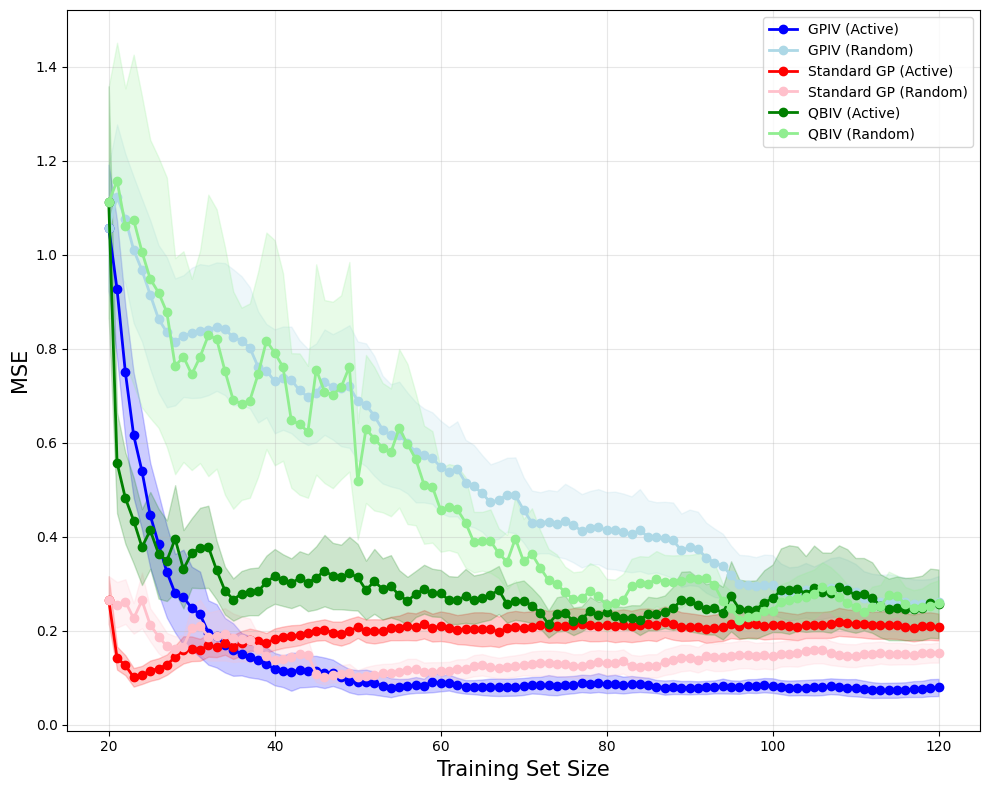}
    \hfill
    \includegraphics[width=0.475\textwidth]{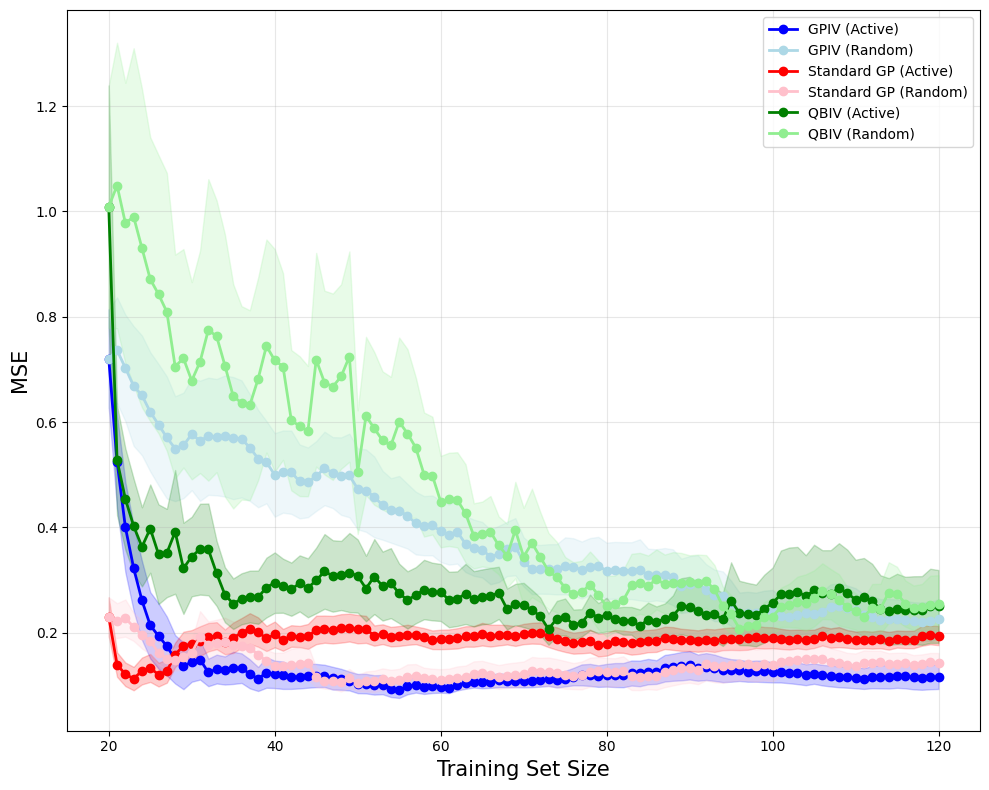}
    \caption{Active Learning of log (left) and linear (right) design.}
    \label{activeiv}
\end{figure}

From the figure \ref{activeiv}, we show that our method improved faster than our competitors and random learning (randomly adding a new data to the model), which means our posterior mean is informative on choosing the data with the highest potential to improve our model predictions.

\subsection{Discussion on the Lengthscales}
\label{appendix: abstudy lengthscale}
\begin{table}[]
\centering
\caption{Effect of optimized lengthscales on MSE, Coverage, and Area under AUC with quantile $=0.75$ for demand design with 200 data.}
\label{ablens}
\begin{tabular}{lccc}
\toprule
\textbf{lengthscale optimized} & MSE & Coverage & AUC \\
\midrule
$(p,t,s)$               & .070 (.003) & .961 (.007) & .918 (.003) \\
$(p,t,s,c)$             & .275 (.030) & .948 (.007) & .923 (.003) \\
$(t,s)$                 & .546 (.020) & .957 (.004) & .931 (.002) \\
$(c,t,s)$               & .724 (.023) & .930 (.005) & .915 (.002) \\
$\emptyset$             & .642 (.033) & .851 (.005) & .895 (.002) \\
$(c)$                   & .906 (.043) & .842 (.006) & .866 (.003) \\
$(p)$                   & .206 (.014) & .788 (.006) & .821 (.004) \\
\bottomrule
\end{tabular}
\end{table}

The ablation study in Table \ref{ablens} demonstrates that the choice of dimensions for lengthscale optimization significantly impacts model performance on the demand dataset ($n=200$). Optimizing the lengthscales for dimensions $x=(p, t, s)$ yields the most effective overall results, achieving the lowest MSE, near-nominal coverage probability, and a high AUC. Notably, in both synthetic and demand design settings (both for IV and Proxy), optimizing the lengthscale for $z$ results in an extremely small estimated value, leading to excessively rough function estimates and consequently poor predictive and uncertainty quantification performance. The results further indicate that optimizing only component $c$ of $z$—rather than using a median heuristic—produces even worse outcomes.

In additional to the demand setting, we further conducted the sensitivity analysis on the IV synthetic data. We observed that, in figure \ref{opt_lz}, the lengthscale parameter of \( Z \), denoted as \( l_z \), tends to shrink to very small values (close to zero) across all designs. This lengthscale is highly sensitive to the data, and such instability is detrimental to predictive performance, as an excessively small \( l_z \) can lead to overfitting. To address this issue, we recommend fixing \( l_z \) using the median heuristic.

Optimizing \(l_x\) and \(\sigma\) while fixing \(l_z\) and \(\eta\) indeed improves MSE (figure \ref{sensi_lx}). However, also optimizing \(l_z\) often leads to poor estimation, as the lengthscale tends to zero (figure \ref{sensi_lz}). Our tests show that the estimation is not highly sensitive to \(l_z\); setting it via the median heuristic works well.

\begin{figure}[htbp]  
  \centering
  \includegraphics[width=0.7\columnwidth]{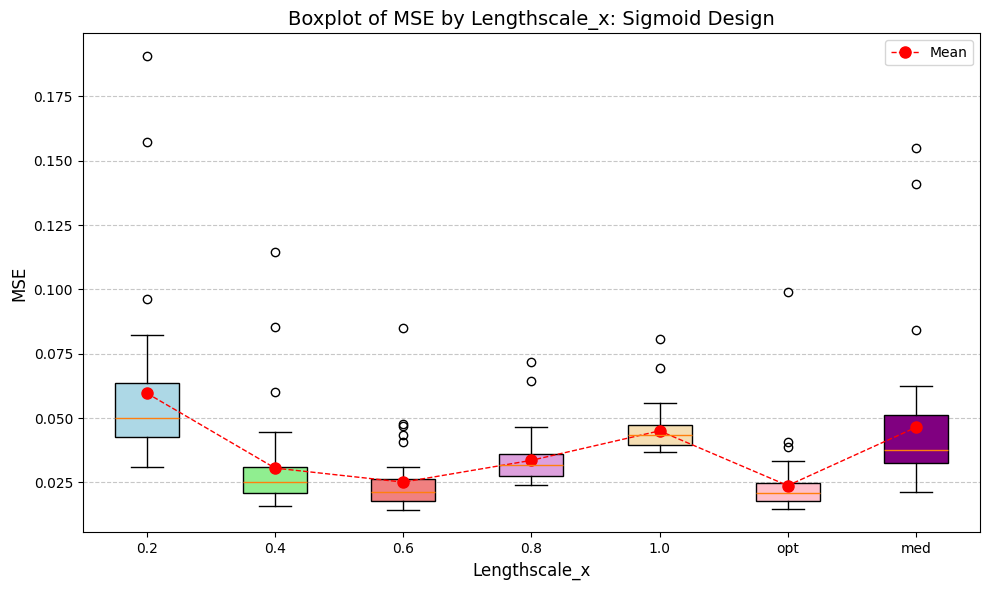}
  \caption{Sensitivity analysis for various lengthscale of x for the log design with $n =2000$}
  \label{sensi_lx}
\end{figure}

\begin{figure}[htbp]  
  \centering
  \includegraphics[width=0.7\columnwidth]{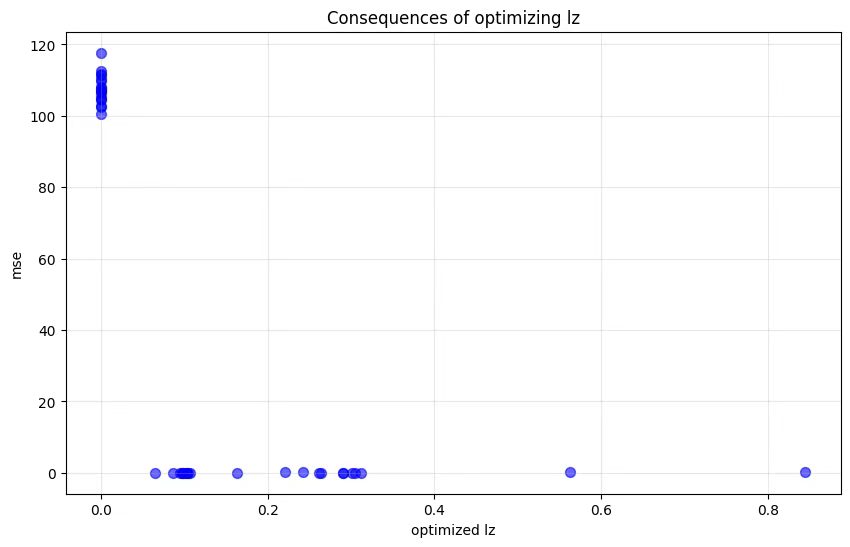}
  \caption{Consequence for optimizing lengthscale of z for the log design with $n =2000$: plots of MSE versus optimized $l_z$.}
  \label{opt_lz}
\end{figure}

\begin{figure}[htbp]  
  \centering
  \includegraphics[width=0.7\columnwidth]{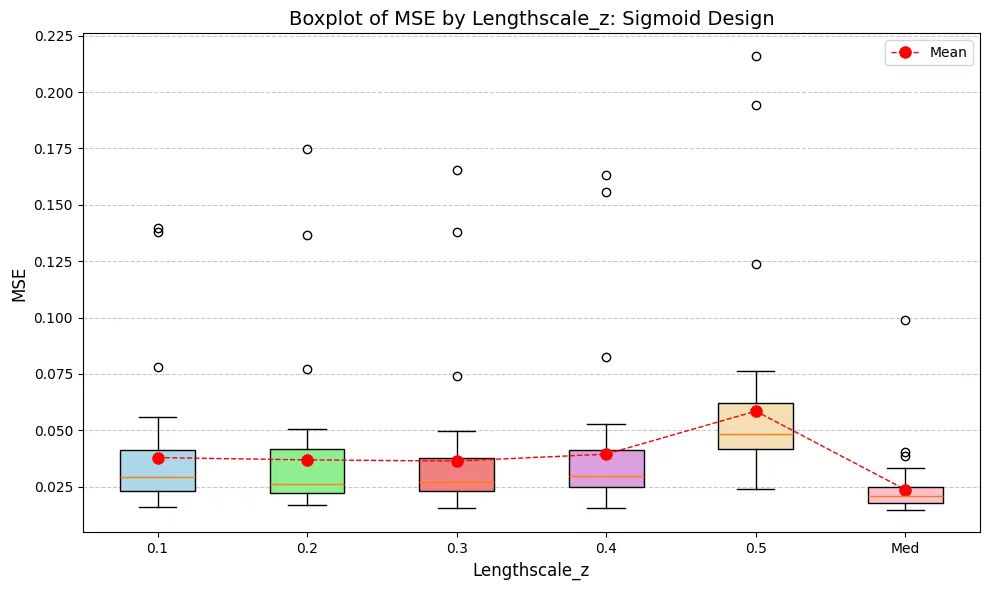}
  \caption{Sensitivity analysis for various lengthscale of z for the log design with $n =2000$}
  \label{sensi_lz}
\end{figure}


\subsection{Sensitivity Analysis Study - $\eta$}
\label{appendix: eta}
\begin{figure}[]
    \centering
    \includegraphics[width=0.9\textwidth]{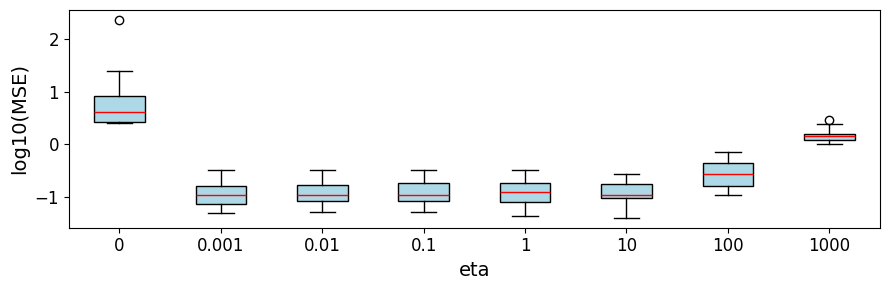}
    \caption{Sensitivity Analysis on $\eta$}
    \label{ablationeta}
\end{figure}

We conduct a sensitivity analysis on the regularization hyperparameter $\eta$. We test $\eta$ over a logarithmic grid ranging from $0.001$ to $1000$, including $\eta = 0$, using $200$ data points. From Figure \ref{ablationeta}, we observe that the MSE remains quite similar for $\eta$ values between $0.001$ and $1$, and deteriorates as regularization increases beyond this range. Moreover, the model also performs poorly when $\eta$ is very close to or exactly $0$. Therefore, we recommend choosing $\eta$ in the range $[0.01, 1]$, or setting $\eta = 0.001 \cdot n$ if the data are standardized.

\subsection{Scalability and approximations}
Without further matrix approximation for scalability, all of the methods require $\mathcal{O}(n^3)$. The mainly difference would be the data-splitting and the hyperparameter tuning methods.

We can use the random feature approximation \citet{liu2021random} or the Nystrom approximation \cite{zhang2023instrumental} to decrease the computational time from $\mathcal{O}(n^3)$ to $\mathcal{O}(n^2+m^2n)$, where $m$ is the number of features used.

We performed a further simulation on the IV synthetic design with large sample $(n=5000)$. For simplicity, we optimize only the $\sigma^2$ and set the remaining hyperparameters heuristically, as the training time is sensitive to the optimzation settings, including the learning rate and the number of iterations, especially when multiple hyperparameters are optimized jointly. From the table \ref{tab:mse_large_n}, we conclude that for large data size, our method provides the lowest MSE. Even if we split our data, the MSE is still the second lowest. One trade-off of this is that we might have a larger time cost (about three times larger than QBIV).

We did not use approximation methods for all methods for the comparison of accuracy. QBIV and KIV use split datasets, which leads to a smaller computational cost compared to our cases. If we split our data, then the computational time would be similar. MMRIV uses one-stage estimation with the lowest computational time, but again, it does not provide uncertainty quantification, and bootstrap with it and KIV would be extremely time consuming, compared to our method which provides UQ simultaneously

\begin{table}[htbp]
\centering
\caption{Large sample ($n=5000$) MSE for synthetic designs (standard errors in parentheses). The best result in each row is highlighted in \textcolor{red}{red} and the second best in \textcolor{orange}{orange}.}
\label{tab:mse_large_n}
\adjustbox{max width=\textwidth}{
\begin{tabular}{l c c c c c}
\toprule
Design & \textbf{GPIV} & \textbf{GPIV (Split)} & QBIV & KIV & MMRIV \\
\midrule
log   & \textcolor{red}{.0306(.0076)} & \textcolor{orange}{.0359(.0102)} & .0363(.0069) & .0534(.0096) & .0454(.0078) \\
sine  & \textcolor{red}{.0273(.0073)} & \textcolor{orange}{.0368(.0092)} & .0545(.0133) & .0505(.0067) & .0722(.0144) \\
linear& \textcolor{red}{.0194(.0055)} & \textcolor{orange}{.0294(.0074)} & .0307(.0124) & .0317(.0133) & .0402(.0069) \\
\bottomrule
\end{tabular}
}
\end{table}

\begin{table}[htbp]
\centering
\caption{Average seconds to perform one estimation on synthetic data for different sample sizes $n$.}
\label{tab:runtime}
\adjustbox{max width=\textwidth}{
\begin{tabular}{l c c c c c}
\toprule
 $n$ & \textbf{GPIV} & \textbf{GPIV (Split)} & QBIV & KIV & MMRIV \\
\midrule
$200$   & 0.123 & 0.063 & 0.052 & 0.107 & 0.047 \\
$2000$  & 34.43 & 15.56 & 10.01 & 20.47 & 9.46 \\
$5000$  & 303.62 & 141.04 & 105.40 & 181.98 & 82.80 \\
\bottomrule
\end{tabular}
}
\end{table}

\subsection{Ablation studies on confounding power}

We modify the strength of confounding, $\rho$ in the Demand setting. We follow the KIV paper's design - let $\rho = 0.1,0.5,0.9$ be the weak, medium, strong confounding. We show that the MSE and the UQ performance did not change much, compared with the main study. 
\begin{table}[htbp]
\centering
\caption{MSE (with standard error in parentheses) for different IV methods across sample sizes and confounding strength. The best result in each row is highlighted in \textcolor{red}{red} and the second best in \textcolor{orange}{orange}.}
\label{tab:iv_results}
\adjustbox{max width=\textwidth}{
\begin{tabular}{l l *{4}{l}}
\toprule
\multicolumn{1}{c}{$n$} & \multicolumn{1}{c}{Confounding} & \multicolumn{1}{c}{\textbf{GPIV}} & \multicolumn{1}{c}{KIV} & \multicolumn{1}{c}{MMRIV} & \multicolumn{1}{c}{QBIV} \\
\midrule
\multirow{3}{*}{200} 
 & Weak   & \textcolor{red}{.0812(.0039)} & .7126(.0303) & .6481(.0325) & \textcolor{orange}{.6313(.0362)} \\
 & Medium & \textcolor{red}{.0702(.0032)} & .7133(.0318) & .6493(.0257) & \textcolor{orange}{.6320(.0264)} \\
 & Strong & \textcolor{red}{.0723(.0270)} & .7129(.0301) & .7131(.0268) & \textcolor{orange}{.6810(.0256)} \\
\addlinespace
\multirow{3}{*}{500} 
 & Weak   & \textcolor{red}{.0515(.0026)} & .6051(.0298) & .5960(.0192) & \textcolor{orange}{.5178(.0189)} \\
 & Medium & \textcolor{red}{.0533(.0029)} & .6046(.0300) & .5822(.0148) & \textcolor{orange}{.5203(.0163)} \\
 & Strong & \textcolor{red}{.0526(.0031)} & .6058(.0322) & .5970(.0157) & \textcolor{orange}{.5179(.0182)} \\
\addlinespace
\multirow{3}{*}{1000}
 & Weak   & \textcolor{red}{.0266(.0021)} & .4739(.0392) & .5391(.0129) & \textcolor{orange}{.4282(.0177)} \\
 & Medium & \textcolor{red}{.0282(.0012)} & .4742(.0302) & .5389(.0132) & \textcolor{orange}{.4294(.0152)} \\
 & Strong & \textcolor{red}{.0260(.0014)} & .4737(.0263) & .5645(.0097) & \textcolor{orange}{.4400(.0173)} \\
\bottomrule
\end{tabular}
}
\end{table}

\begin{table}[htbp]
\centering
\caption{Coverage (closer to 0.95 the better) and AUC (higher the better) with standard errors in parentheses. The best result in each row is highlighted in \textcolor{red}{red} and the second best in \textcolor{orange}{orange}. Leading zeros are omitted.}
\label{tab:coverage_auc}
\adjustbox{max width=\textwidth}{
\begin{tabular}{l l c c c c}
\toprule
Metric & Confounding & \textbf{GPIV} & KIV(BS) & GPIV(BS) & QBIV \\
\midrule
\multirow{3}{*}{Coverage} 
 & Weak   & \textcolor{red}{.9127(.0077)} & .4895(.0198) & .5948(.0609) & \textcolor{orange}{.8917(.0417)} \\
 & Medium & \textcolor{red}{.9611(.0072)} & .4607(.0293) & .6254(.0588) & \textcolor{orange}{.9031(.0102)} \\
 & Strong & \textcolor{red}{.9531(.0069)} & .4859(.0252) & .5808(.0612) & \textcolor{orange}{.8916(.0106)} \\
\addlinespace
\multirow{3}{*}{AUC} 
 & Weak   & \textcolor{orange}{.8731(.0279)} & .6722(.0080) & .7156(.0126) & \textcolor{red}{.8934(.0020)} \\
 & Medium & \textcolor{red}{.9177(.0173)} & .6635(.0081) & .7072(.0119) & \textcolor{orange}{.9044(.0012)} \\
 & Strong & \textcolor{orange}{.8925(.0237)} & .6724(.0125) & .7013(.0086) & \textcolor{red}{.9034(.0017)} \\
\bottomrule
\end{tabular}
}
\end{table}

\subsection{Discussion on Weak IV}
To test the scenario of weak IV, we modified our data generating process in the synthetic design. We decrease $\alpha$ from $0.5$ to $0.1$ where our IV $Z = \Phi(W)$, $X = \Phi(\alpha W+(1-\alpha)V)$.

We see that the MSE (Tables \ref{tab:mse_weakiv}-\ref{tab:uqweakiv}) increases across all the designs, but our method is still comparable to other baselines. More importantly, our methods still provide reliable UQ (on coverage and the accuracy rejection curve).

\begin{table}[htbp]
\centering
\caption{MSE (mean with standard error) for different designs and sample sizes. The best result in each row is highlighted in \textcolor{red}{red} and the second best in \textcolor{orange}{orange} (lower is better).}
\label{tab:mse_weakiv}
\adjustbox{max width=\textwidth}{
\begin{tabular}{l l c c c c}
\toprule
Design & $n$ & \textbf{GPIV} & KIV & MMRIV & QBIV \\
\midrule
\multirow{3}{*}{log} 
 & 200 & \textcolor{orange}{1.2285(.1260)} & 1.4408(.2012) & \textcolor{red}{.9129(.0781)} & 1.4294(.2636) \\
 & 500 & 1.0359(.0998) & 1.1758(.1517) & \textcolor{red}{.8711(.0501)} & \textcolor{orange}{1.0220(.0744)} \\
 & 1000 & \textcolor{orange}{.8648(.0592)} & .9159(.0742) & \textcolor{red}{.8600(.0263)} & .9738(.0903) \\
\addlinespace
\multirow{3}{*}{sine} 
 & 200 & \textcolor{orange}{.7516(.0998)} & .8149(.0863) & \textcolor{red}{.6724(.0427)} & .8129(.0392) \\
 & 500 & \textcolor{orange}{.2674(.0313)} & .4819(.0524) & \textcolor{red}{.2587(.0267)} & .3734(.0328) \\
 & 1000 & \textcolor{red}{.1330(.0152)} & .4100(.0299) & \textcolor{orange}{.2367(.0198)} & .3718(.0219) \\
\addlinespace
\multirow{3}{*}{linear} 
 & 200 & 1.0618(.0646) & \textcolor{orange}{.8922(.0856)} & \textcolor{red}{.7364(.0378)} & 1.2187(.2776) \\
 & 500 & \textcolor{orange}{.7863(.0605)} & .8145(.0615) & \textcolor{red}{.6851(.0293)} & .8432(.0648) \\
 & 1000 & \textcolor{orange}{.7495(.0614)} & .8012(.0346) & \textcolor{red}{.6647(.0133)} & .7855(.0629) \\
\bottomrule
\end{tabular}
}
\end{table}

\begin{table}[htbp]
\centering
\caption{Coverage (closer to $0.95$ the better) and AUC (higher the better) with standard errors in parentheses. The best result in each row is highlighted in \textcolor{red}{red} and the second best in \textcolor{orange}{orange}.}
\label{tab:uqweakiv}
\adjustbox{max width=\textwidth}{
\begin{tabular}{l l c c c c}
\toprule
Design & Metrics & \textbf{GPIV} & \textbf{QBIV} & \textbf{GPIV(BS)} & \textbf{KIV(BS)} \\
\midrule
\multirow{2}{*}{Log} 
 & Coverage & \textcolor{red}{.8564(.0175)} & .7572(.0282) & .4206(.0322) & \textcolor{orange}{.8034(.0220)} \\
 & AUC & \textcolor{red}{.9403(.0091)} & \textcolor{orange}{.9074(.0120)} & .5638(.0243) & .7728(.0270) \\
\addlinespace
\multirow{2}{*}{Sine} 
 & Coverage & \textcolor{red}{.9564(.0163)} & \textcolor{orange}{.9738(.0140)} & .5118(.0495) & .3232(.0161) \\
 & AUC & \textcolor{orange}{.8696(.0211)} & .8481(.0239) & .8689(.0225) & \textcolor{red}{.8821(.0198)} \\
\addlinespace
\multirow{2}{*}{Linear} 
 & Coverage & \textcolor{red}{.9118(.0457)} & \textcolor{orange}{.8410(.0395)} & .5818(.0136) & .7888(.0236) \\
 & AUC & \textcolor{red}{.9398(.0072)} & \textcolor{orange}{.9042(.0132)} & .7541(.0203) & .7732(.0253) \\
\bottomrule
\end{tabular}
}
\end{table}

\subsection{Discussion on Weak identification}
We investigate violations of the identification assumption in synthetic experiments by relaxing the independence between the error and the IV. Specifically, we set the covariance between \(e\) and \(V\) to \(0.8\), and vary the covariance between \(e\) and \(W\) (denoted \(\theta\), where \(Z = \Phi(W)\)) from \(0\) to \(0.5\). Results show that when \(\theta\) is small (i.e., mild violation), MSE and coverage do not deteriorate (Table \ref{tab:mse_coverage_ivviolation}, Fig \ref{IVViolate}). However, when the violation is moderately large (\(\theta = 0.3\)), MSE increases and coverage drops significantly, indicating failure to recover the true ATE. This is expected as a strong violation of causal assumptions implies the causal estimand is no longer identifiable (estimatable) using observational data.

\begin{table}[htbp]
\centering
\caption{MSE and Coverage (with standard errors in parentheses) for different $\theta$ (covariance between IV and error) values across synthetic designs.}
\label{tab:mse_coverage_ivviolation}
\adjustbox{max width=\textwidth}{
\begin{tabular}{l l c c c c c c}
\toprule
\multirow{2}{*}{Design} & \multirow{2}{*}{Metric} & \multicolumn{6}{c}{$\theta$} \\
\cmidrule(lr){3-8}
 & & 0 & 0.1 & 0.2 & 0.3 & 0.4 & 0.5 \\
\midrule
\multirow{2}{*}{log} 
 & MSE & .1432(.0210) & .1064(.0128) & .0954(.0082) & .1159(.0075) & .1653(.0093) & .2611(.0125) \\
 & Cov & .9993(.0007) & 1.0000(.0000) & 1.0000(.0000) & .9932(.0047) & .9545(.0125) & .8296(.0228) \\
\addlinespace
\multirow{2}{*}{sine} 
 & MSE & .1633(.0200) & .1600(.0156) & .2369(.0212) & .3552(.0299) & .5028(.0269) & .7323(.0291) \\
 & Cov & .9483(.0205) & .9660(.0125) & .9149(.0167) & .8419(.0270) & .8101(.0228) & .7716(.0273) \\
\addlinespace
\multirow{2}{*}{linear} 
 & MSE & .1464(.0201) & .1266(.0115) & .1412(.0090) & .1895(.0116) & .2750(.0137) & .4226(.0159) \\
 & Cov & .9889(.0046) & .9856(.0098) & .9535(.0219) & .8436(.0308) & .6705(.0288) & .5096(.0179) \\
\bottomrule
\end{tabular}
}
\end{table}

\begin{figure}[htbp]  
  \centering
  \includegraphics[width=0.7\columnwidth]{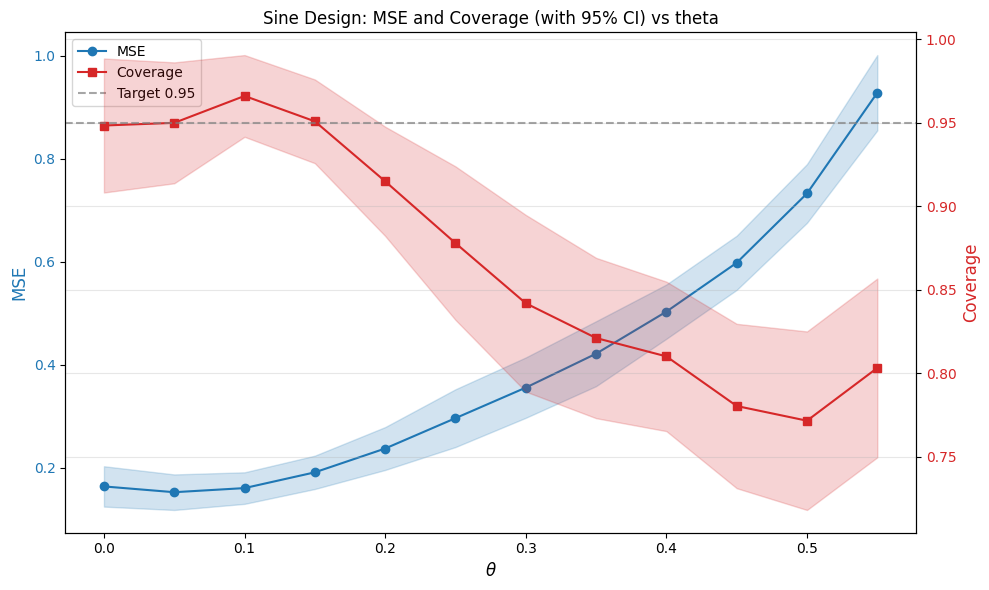}
  \caption{Plots of MSE and Coverage with $95\%$ C.I. for $\theta$ from $0.00$ to $0.55$}
  \label{IVViolate}
\end{figure}

\subsection{Proxy Experiments Settings}
\label{appendix:ProxyExp}
\textbf{Synthetic data.} \quad
Here are the details of data generating process:
\begin{align*}
U &:= [U_1, U_2], \quad U_2 \sim \mathrm{Uniform}[-1, 2] \\
U_1 &\sim \mathrm{Uniform}[0, 1] - \mathbf{1}[0 \leq U_2 \leq 1] \\
W &:= [W_1, W_2] = [U_1 + \mathrm{Uniform}[-1, 1], U_2 + \varepsilon_1] \\
Z &:= [Z_1, Z_2] = [U_1 + \varepsilon_2,U_2 + \mathrm{Uniform}[-1, 1]] \\
X &:= U_2 + \varepsilon_3 \\
Y &:= 3 \cos\bigl(2(0.3U_1 + 0.3U_2 + 0.2) + 1.5A\bigr) + \varepsilon_4
\end{align*}
such that $\varepsilon_{1,2,3,4} \sim N(0,1)$ independently.

\begin{figure}[]
    \centering
    \includegraphics[width=0.475\textwidth]{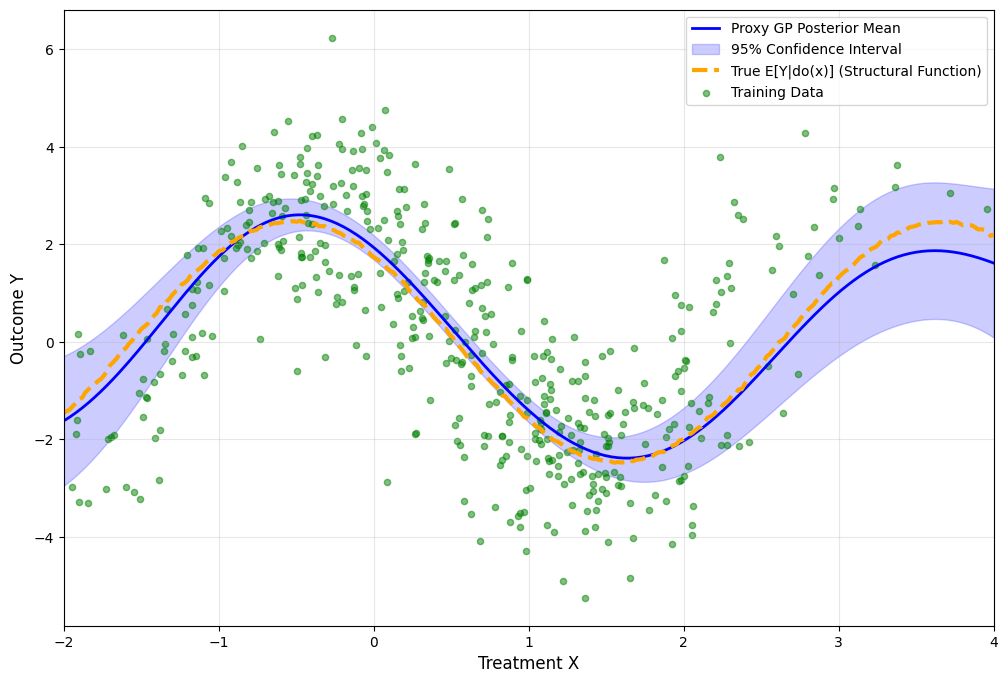}
    \hfill
    \includegraphics[width=0.475\textwidth]{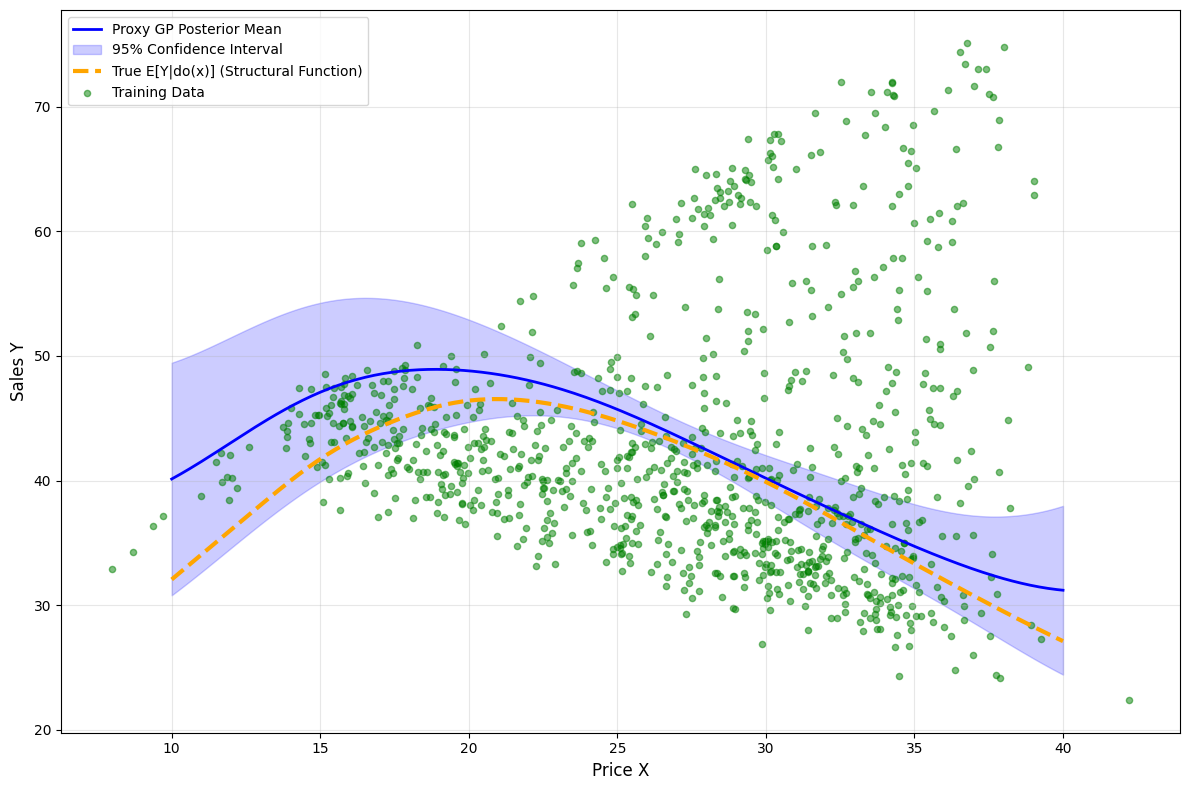}
    \caption{Demo of Proximal synthetic (left) and demand (right) design with 1000 training data.}
    \label{demoproxy}
\end{figure}

\begin{figure}[]
    \centering
    \includegraphics[width=0.95\textwidth]{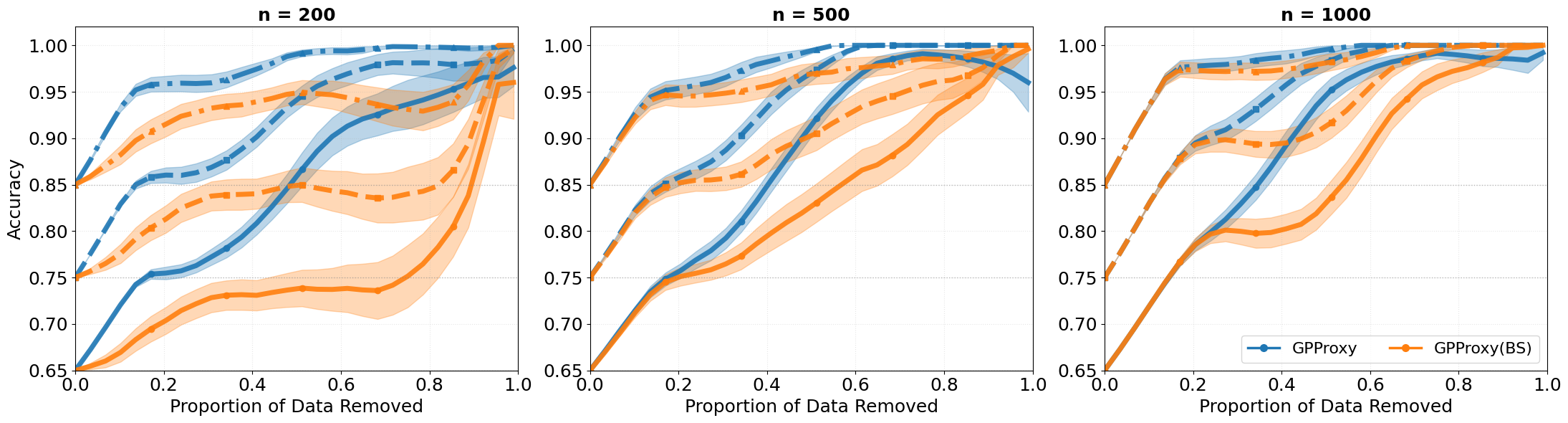}
    \caption{Accuracy-rejection curve for proxy synthetic design (data size $=200,500,1000$), with quantile $= 0.65,0.75,0.85$.}
    \label{arcproxysyn}
\end{figure}

\begin{figure}[]
    \centering
    \includegraphics[width=0.95\textwidth]{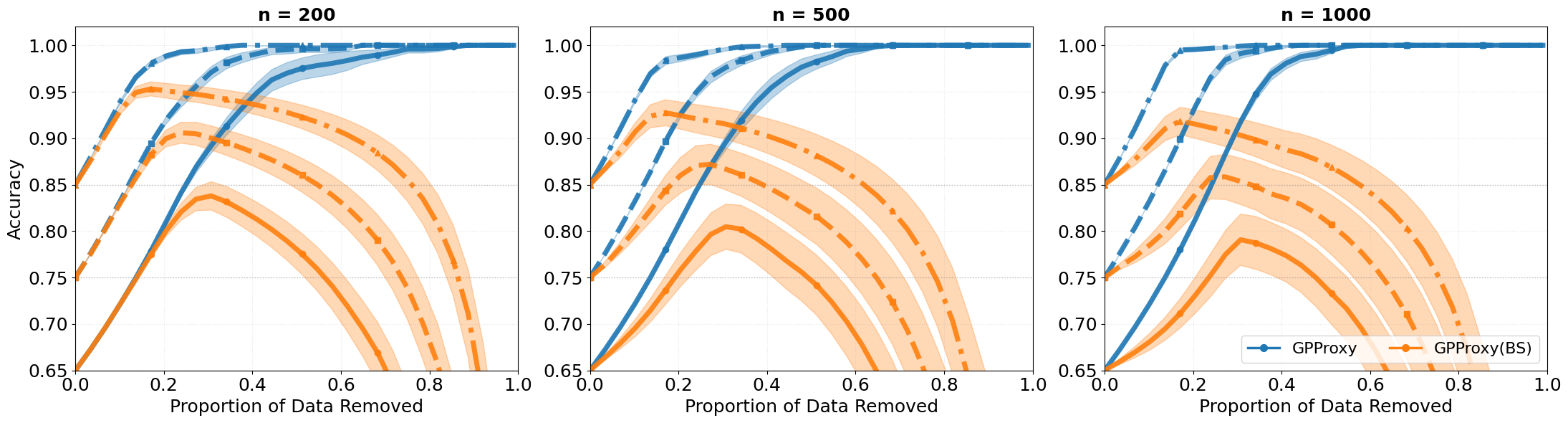}
    \caption{Accuracy-rejection curve for proxy demand design (data size $=200,500,1000$), with quantile $= 0.65,0.75,0.85$.}
    \label{arcproxydemand}
\end{figure}

\textbf{Demand data.} \quad
Structural equations model real-world pricing with unobserved demand sensitivity.
\begin{align*}
&\text{Demand: } U \sim \mathcal{U}(0, 10), \\
&\text{Fuel cost: } [Z_1, Z_2] = \bigl[2\sin(2\pi U/10) + \varepsilon_1,\; 2\cos(2\pi U/10) + \varepsilon_2 \bigr], \\
&\text{Web views: } W = 7g(U) + 45 + \varepsilon_3, \\
&\text{Price: } X = 35 + (Z_1 + 3)g(U) + Z_2 + \varepsilon_4, \\
&\text{Sales: } Y = X \times \min\left(\exp\left(\frac{W-X}{10}\right),\; 2\right) - 5g(U) + \varepsilon_5.
\end{align*}
where
\[
g(u) = 2\left( \frac{(u{-}5)^4}{600} + e^{-4(u{-}5)^2} + \frac{u}{10} - 2 \right),
\quad \varepsilon_i \overset{\mathrm{i.i.d.}}{\sim} \mathcal{N}(0,1).
\]

Note that we slightly modify the Sales function in \citet{xu2021deep} to make the ATE more nonlinear.

\subsection{Implementation Details}
In practice, jointly optimizing all hyperparameters is not always feasible due to potential nonidentifiability. For both the IV and Proxy settings, we observed that optimizing the lengthscale of $Z$ tends to shrink it to a very small value, leading to extremely poor predictive performance. Therefore, we recommend fixing it using the median heuristic. Moreover, since $\eta$ serves as a penalization constant, direct optimization via marginal likelihood is not meaningful; we instead set it to a default value of $\eta = 0.1$ in both GPIV and GPProxy, as suggested in the section \ref{appendix: eta}. 

We standardize the inputs and optimize the lengthscales of $X$ and $W$, initializing them with the median heuristic. The noise variance $\sigma^2$ is also optimized, with an initial value of $0.25$.

In the IV setting, we implement the baseline methods KIV \citep{singh2019}, MMRIV \citep{zhang2023instrumental}, and QBIV \citep{wang2021quasi}. For the proximal setting, we include KPV, PMMR \citep{mastouri2021proximal}, PKIPW, PKDR \citep{wu2023doubly}, KNC \citep{singh2023kernelmethodsunobservedconfounding}, and KNC-orig (which corresponds to DProxy). For all methods, we adopt the default settings from the original source code: only the regularization hyperparameters are tuned, while the lengthscales are fixed using the median heuristic. Note that the code for KNC and KNC-orig are not publicly available yet; we therefore implement them in Python following the instructions provided in the experimental sections of the respective papers. We re-implement KIV and MMRIV in Python based on their original MATLAB code. In our bootstrap implementation, we used 25 samples per iteration and confirmed that increasing the number of bootstrap iterations does not compromise uncertainty quantification performance. This finding aligns with the conclusions of \citet{wang2021quasi}.

\subsection{Wilcoxon Signed-rank test}
We employ the Wilcoxon signed-rank test at the $5\%$ significance level to compare the methods. Table \ref{wilcoxon} lists the methods that perform significantly worse than our approach.
\begin{table}[htbp]
\centering
\caption{Methods with significantly worse performancen on MSE than GPIV (left) and GPProxy (right) under Wilcoxon signed-rank test, $5\%$.}
\label{wilcoxon}
\begin{minipage}[]{0.485\textwidth}
\centering
\small
\begin{tabular}{lcl}
\toprule
Design & $n$ & Methods \\
\midrule
\multirow{3}{*}{sine}   & 200 & KIV \\
                        & 500 & KIV, MMRIV \\
                        & 1000 & KIV, MMRIV \\
\addlinespace
\multirow{3}{*}{log}    & 200 & KIV \\
                        & 500 & KIV, QBIV \\
                        & 1000 & KIV, MMRIV, QBIV \\
\addlinespace
\multirow{3}{*}{linear} & 200 & KIV \\
                        & 500 & QBIV \\
                        & 1000 & KIV, MMRIV, QBIV \\
\addlinespace
\multirow{3}{*}{demand} & 200 & KIV, MMRIV, QBIV \\
                        & 500 & KIV, MMRIV, QBIV \\
                        & 1000 & KIV, MMRIV, QBIV \\
\bottomrule
\end{tabular}
\end{minipage}%
\hfill
\begin{minipage}[]{0.485\textwidth}
\centering
\small
\begin{tabular}{lcl}
\toprule
Design & $n$ & Methods \\
\midrule
\multirow{3}{*}{synthetic} & 200 & -- \\
                           & 500 & -- \\
                           & 1000 & -- \\
\addlinespace
\multirow{3}{*}{demand}    & 200 & -- \\
                           & 500 & DProxy, KNC \\
                           & 1000 & DProxy, PMMR, PKDR, KNC \\
\bottomrule
\end{tabular}
\end{minipage}

\vspace{2mm}
\footnotesize\textit{Note:} For GPIV, the ``demand'' design uses normalized MSE. ``--'' indicates no method significantly worse than the our method.
\end{table}

\end{document}